\documentclass{article}

\usepackage[final,nonatbib]{neurips_2021}

\usepackage[utf8]{inputenc} % allow utf-8 input
\usepackage[T1]{fontenc}    % use 8-bit T1 fonts
\usepackage{hyperref}       % hyperlinks
\usepackage{url}            % simple URL typesetting
\usepackage{booktabs}       % professional-quality tables
\usepackage{amsfonts}       % blackboard math symbols
\usepackage{nicefrac}       % compact symbols for 1/2, etc.
\usepackage{microtype}      % microtypography
\usepackage{xcolor}         % colors
\usepackage[linesnumbered,lined,ruled,vlined,noend]{algorithm2e}

\usepackage{graphicx}
\usepackage{subfigure}
\usepackage{amsmath}
\usepackage{amsthm}

\usepackage{mathtools}

\newcommand{\abs}[1]{\left| #1\right|}

\newcommand{\br}[1]{\left\{#1\right\}}

\newcommand{\REAL}{\ensuremath{\mathbb{R}}}

\newcommand{\eps}{\varepsilon}

\newcommand{\OPT}{\mathrm{opt}}
\newcommand{\cost}{\mathrm{\ell}}

\newcommand{\K}{\mathcal{K}}
\newcommand{\B}{\mathcal{B}}
\newcommand{\E}{\mathcal{E}}

\newcommand{\coresetAlg}{\textsc{Signal-Coreset}}
\newcommand{\coresetQuery}{\textsc{Fitting-loss}}

\newcommand{\partitionAlg}{\textsc{Partition}}
\newcommand{\partitionAlgOneD}{\textsc{SlicePartition}}
\newcommand{\bicriteriaAlg}{\textsc{Bicriteria}}

\newcommand{\DTcoreset}{\texttt{DT-coreset}}
\newcommand{\randomSample}{\texttt{RandomSample}}

\newcommand{\randomForest}{\texttt{RandomForestRegressor}}
\newcommand{\lightGBM}{\texttt{LGBMRegressor}}

\renewcommand{\S}{\mathrm{SEG}}

\renewcommand{\paragraph}[1]{\medskip\noindent\textbf{{#1} }}

\newcommand{\vertiii}[1]{{\left\vert\kern-0.25ex\left\vert\kern-0.25ex\left\vert #1
    \right\vert\kern-0.25ex\right\vert\kern-0.25ex\right\vert}}

\newtheorem{theorem}{Theorem}
\newtheorem{lemma}[theorem]{Lemma}
\newtheorem{observation}[theorem]{Observation}

\newtheorem{definition}[theorem]{Definition}
\newtheorem{corollary}[theorem]{Corollary}
\newtheorem{claim}{Claim}[theorem]

\newif\ifproofs
\proofstrue%

\DeclareMathOperator*{\argmin}{arg\,min}

\title{Coresets for Decision Trees of Signals}

\author{
    Ibrahim Jubran, Ernesto Sanches, Ilan Newman, Dan Feldman\\
    University of Haifa, Israel\\
    {\tt\small \{ibrahim.jub,  ernestosanches, dannyf.post\}@gmail.com}\\
    {\tt\small \{ilan\}@cs.haifa.ac.il}
    }

\begin{document}

\maketitle

\begin{abstract}
A $k$-decision tree $t$ (or $k$-tree) is a recursive partition of a matrix (2D-signal) into $k\geq 1$ block matrices (axis-parallel rectangles, leaves) where each rectangle is assigned a real label. Its regression or classification loss to a given matrix $D$ of $N$ entries (labels) is the sum of squared differences over every label in $D$ and its assigned label by $t$.
Given an error parameter $\varepsilon\in(0,1)$, a $(k,\varepsilon)$-coreset $C$ of $D$ is a small summarization that provably approximates this loss to \emph{every} such tree, up to a multiplicative factor of $1\pm\varepsilon$. In particular, the optimal $k$-tree of $C$ is a $(1+\varepsilon)$-approximation to the optimal $k$-tree of $D$.

We provide the first algorithm that outputs such a $(k,\varepsilon)$-coreset for \emph{every} such matrix $D$. The size $|C|$ of the coreset is polynomial in $k\log(N)/\varepsilon$, and its construction takes $O(Nk)$ time.
This is by forging a link between decision trees from machine learning -- to partition trees in computational geometry. 

Experimental results on \texttt{sklearn} and \texttt{lightGBM} show that applying our coresets on real-world data-sets boosts the computation time of random forests and their parameter tuning by up to x$10$, while keeping similar accuracy. Full open source code is provided.
\end{abstract}

\section{Introduction} \label{sec:intro}
% \begin{enumerate}
%     \item Related works by Sudipto (BIRCH) (Reviewer ybbQ).
%     \item Related work on databases (Reviewer nDd9).
%     \item ``which of the techniques are entirely new??''
% \end{enumerate}
Decision trees are one of the most common algorithms used in machine learning today, both in the academy and industry, for classification and regression problems~\cite{rokach2005decision}. 
% While there does not seem to be an agreed formal definition in existing papers, 
Informally, a decision tree is a recursive binary partition of the input feature space into hyper-rectangles, where each such hyper-rectangle is assigned a label. 
% The label may be discrete (as in classification) or continuous (regression trees). 
If the labels are given from a discrete set, the trees are usually called \emph{classification trees}, and otherwise they are usually called \emph{regression trees}.
Variants include non-binary partitions and forests \cite{ho1995random}.%, which are an ensemble of trees with some voting rule between them.

\begin{figure}[h]
  \centering
  \includegraphics[width=0.9\textwidth]{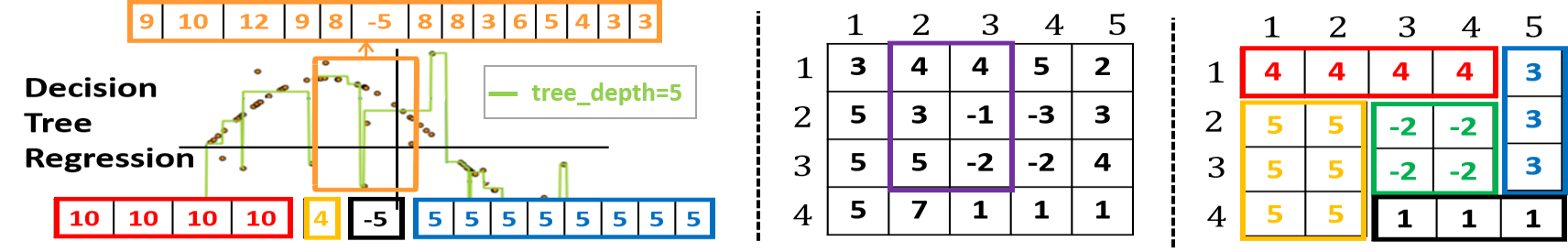}
  \caption{\textbf{(Left): }A one dimensional signal (orange points) and its segmentation into $25$ ``smooth'' segments / leaves (green lines). Image taken from Section 1.10 (``Decision Trees'') of the sklearn's User Guide~\cite{sklearnDT}. The vector $v$ on top represents a subset of the signal's values. The bottom vectors represent a $4$-segmentation of $v$, similar to the horizontal green line segments. Each segment contains the average value of its corresponding segment from $v$. \textbf{(Middle): }A matrix that represents the $4\times 5$ signal $D = \br{((1,1),3), ((1,2),4), ((1,3),5),\cdots}$ (in black) and a $3\times 2$ matrix that represents a $3\times 2$ sub-signal $B = \br{((1,2),4), ((1,3),4),((2,2),3),\cdots}$ (in purple). \textbf{(Right): }A matrix that represents a 5-segmentation $s$ of $A = [4]\times [5]$; see Definition~\ref{def:ksegmentation}. Since $s$ is a $5$-segmentation, it partitions the $[4]\times[5]$ matrix into $5$ distinct block matrices $B_1$ (red), $B_2$ (blue), $B_3$ (yellow), $B_4$ (green), and $B_5$ (black) such that $s$ assigns the same number for all the entries in the same block. The SSE fitting loss $\cost(D,s)$ is the sum of squared differences over every entry in the left matrix to its corresponding entry on the right matrix. Also, there is no $k$-tree that can obtain the same partition.}
  \label{fig:defsPic}
\end{figure}

\textbf{Why decision trees?}
Advantages of decision trees, especially compared to deep learning, include:\\
\textbf{(i) }Interpretability. They are among the most popular algorithms for interpretable (transparent) machine learning~\cite{hu2019optimal}.
\textbf{(ii) }Usually require small memory space, which also implies fast classification time.
\textbf{(iii) }Accuracy. Decision trees are considered as one of the few competitors of deep networks. 
% Even in today's deep learning era, decision trees are considered to be the classifier with higher accuracy, and one of the few competitors of deep networks. 
In competitions, such as the ones in Kaggle~\cite{kaggle}, they are one of the favorite classifiers~\cite{bojer2021kaggle}, especially on small or traditional tabular data.
\textbf{(iv) }May learn from small training data.

\textbf{The goal }is usually to compute the optimal $k$-tree $t^*$ for a given dataset $D$ and a given number $k$ of leaves, according to some given loss function. 
In practice, researchers usually use ensemble of trees called forests, e.g., a Random Forest~\cite{breiman2001random}, which are usually learned from different subsets of the training set. The final classification is then based on a combination rule, such as majority or average vote. Since both the training and classification of each tree are computed independently and possibly in parallel, we focus on the construction of a single tree. 

\textbf{A \emph{dataset}} $D$ in this paper is a set $D=\{(x_1,y_1),\cdots,(x_N,y_N)\} \subseteq A\times\REAL$ of pairs, where $A$ is the \emph{feature space}. Each pair $(x,y)\in D$ consists of a database record (vector / sample) $x \in A$ and its real label $y \in \REAL$. 
As common, we assume that non-real features, such as categorical features, are converted to real numbers; see e.g.~\cite{hardy1993regression}.
For example, in common classification problems $A = \REAL^d$ for some $d\geq 1$ and $y\in\br{0,1}$ is a binary number.
The resulting model may be used for prediction on another test dataset, completion of missing values, or efficient storage of the original dataset by replacing the label $y$ of each pair $(x,y)\in D$ with the label $t(x)$ that was assigned to it by the tree $t$. The last technique is used e.g. in the MPEG4 encoder~\cite{marpe2006h}, where decision trees of a specific structure (quad-trees) are used to compress an image $D$ that consists of pixel-grayscale pairs $(x,y)$.%, where every "image block" is approximated by its average value.

% image blocks with smoother alternatives, as done via quadtrees in modern video compressions e.g., MPEG4~\cite{marpe2006h}.
% More generally, we may want to partition the original dataset into (not necessarily disjoint) sets, aim to compute the optimal decision tree for each one of them, and then combine them into a foreset via majority vote between the trees or other methods~\cite{??}.

\textbf{Challenges. }
The motivation for this paper originated from the following challenges:\\
\textbf{(i) Sub-optimality.} Hardness of decision tree optimization is both a theoretical and practical obstacle~\cite{hu2019optimal}. It is NP-hard to compute the optimal $k$-tree, or its approximation, when the number $k$ is not fixed~\cite{chakaravarthy2007decision, laurent1976constructing}. 
% This is true for most existing definitions of ``decision tree'' and ``approximation''.
There were several attempts to improve the optimality of decision tree algorithms, from binary-split decision trees as in~\cite{bennett1992decision,bertsimas2017optimal}, in a line of work of e.g.~\cite{blanquero2018optimal, verwer2019learning}. % until present (e.g.~\cite{verwer2019learning}).
Nevertheless, greedy implementations e.g., CART~\cite{loh2011classification} and C4.5~\cite{quinlan2014c4} have remained the dominant methods in practice.\\
%Due to lack of guarantees of even closeness to optimality, decision tree algorithms are often greedy or myopic, and sometimes produce unquestionably suboptimal models. 
\textbf{(ii) Computation time.} Due to this lack of optimality, finding a decision tree that provides a good accuracy usually requires many runs, since each of them returns only a local minimum that might be arbitrarily far from the global optimum. The final model usually consists of a forest containing many trees. Popular forest implementations include the famous Sklearn, XGBoost, LightGBM, and CatBoost libraries~\cite{pedregosa2011scikit, chen2016xgboost, ke2017lightgbm, dorogush2018catboost}, which all utilize (as default) an ensemble of at least $100$ trees. 
Moreover, there is a list of dozen parameters to calibrate including: number of trees, depth of each tree, pruning/splitting strategies on each tree and between them, and many others. To this end, the running time for obtaining reasonable results even on moderate size datasets might be impractical.\\ 
\textbf{(iii) Scalability.} Existing techniques tend not to scale to realistically-sized problems unless simplified to trees of a specific form as stated in~\cite{hu2019optimal}.\\%This is not only due to the computation time as stated above, but also due to the fact that existing algorithms require direct access to the input in the RAM memory for a reasonable performance.\\
% \textbf{(iv) Parameter tuning.} The lack of optimality results also in a list of dozens parameters to calibrate including: number of trees, depth of each tree, pruning/splitting strategies on each tree and between them, cost/impurity functions for the nodes/leaves, and many others. Tuning those parameters is usually time consuming.\\
% Using cross validation techniques for verifying these parameters add another serious factor to the running times.\\
\textbf{(iv) Streaming, parallel, and dynamic updates.} The common algorithms mentioned above do not support continuous learning or updating of the modeled tree when an input sample is either added or removed from the dataset, e.g., when the dataset does not fit into memory or arrives on-the-fly. Similarly, we do not know techniques to train a single tree in parallel on multiple machines.
\subsection{Coresets} \label{sec:coresets}
``Coresets are one of the central methods to facilitate the analysis of large data sets.''~\cite{munteanu2018on}.
Informally, for an input dataset $D$, a set $T$ of models, an approximation error $\eps\in(0,1)$, and a loss function $\cost$, a coreset $C$ is a data structure that approximates the loss $\cost(D,t)$ for every model $t\in T$, up to a multiplicative factor of $1\pm\eps$, in time that depends only on $|C|$. Hence, ideally, $C$ is also much smaller than ther original input $D$.

\textbf{Why coresets? }The main motivation for constructing a coreset is to compute the optimal model or its approximation, much faster, while sacrificing little accuracy. Furthermore, a coreset for a family of classifiers is many times a ``silver bullet'' that provides a unified solution to all Challenges (i)-(iv) above. 
Combining the two main coreset properties: merge and reduce~\cite{indyk2014composable, bentley1980decomposable, har2004coresets, agarwal2013mergeable}, which are usually satisfied, with the fact that a coreset approximates every model, and not just the optimal model, enables it to support streaming and distributed data~\cite{braverman2016new,lu2020robust}, parallel computation~\cite{feldman2020core}, handle constrained versions of the problem~\cite{feldman2015more}, model compression~\cite{dubey2018coreset}, parameter tuning~\cite{maalouf2019fast} and more.

\textbf{Coreset construction techniques. }There are many different techniques for coreset construction, ranging from loss-less to lossy, from deterministic to randomized, and from greedy to non-greedy constructions. 
Examples include accurate coresets~\cite{jubranoverview} and non-accurate coresets~\cite{feldman2020turning} via computational geometry, random sampling-based coresets~\cite{feldman2011unified,cohen2015uniform, maalouf2020faster}, and greedy deterministic coresets via the Frank-Wolfe algorithm~\cite{clarkson2010coresets}.
Recently, many works focus on developing frameworks for general families of loss functions, e.g.,~\cite{feldman2011unified, tukan2020coresets}.
We refer the interested reader to the surveys~\cite{agarwal2005geometric,agarwal2013mergeable,phillips2016coresets,braverman2016new,bachem2017practical,feldman2020core} with references therein.

\textbf{Practical usage.} 
Since a coreset is not just another solver that competes with existing solutions, but a data structure for approximating any given model in the family, we can apply \emph{existing} approximation algorithms or heuristics on the coreset to obtain similar results compared to the original (full) data.
Since the coreset is small, we may run these heuristics multiple times, or apply the hyperparameter tuning using the coreset~\cite{maalouf2019fast}, thus reducing the computational burden by orders of magnitude.

\textbf{Main challenges:}
\textbf{(i) }No coreset. Unfortunately, it is not clear at all that a small coreset exists for a given family of models. In fact, we can conclude from~\cite{rosman2014coresets} that a coreset for decision trees does not exist in general; see details below.
In this case, we can either give up on the coreset paradigm and develop a new solver, or add assumption on the input dataset, instead of targeting every possible dataset that may be very artificial and unrealistic, as the counter example in~\cite{rosman2014coresets}. In this paper, we choose the latter option.  
\textbf{(ii) }Unlike, say, uniform sampling, every problem formulation requires a different coreset construction, which may take years of research to design.

\subsection{First coreset for decision trees and their generalization} \label{sec:firstCoreset}

In this paper, we tackle a generalized and more complex set of models than decision trees, where, rather than a recursive binary partition, we allow the input feature space $\REAL^d$ to be partitioned into any $k$ disjoint axis-parallel hyper-rectangles; see Fig.~\ref{fig:defsPic}.
This generalization is essential in order to support future non-recursive and not necessarily binary classification models, e.g., ID3 and C4.5~\cite{hssina2014comparative, quinlan2014c4}. To our knowledge, this is the first coreset with provable guarantees for constructing decision trees. 

\begin{definition}[$k$-segmentation] \label{def:ksegmentation}
% For a set $A \subseteq \REAL^d$ called the \emph{feature space}, and 
For an integer $d \geq 1$ that denotes the dimension of the \emph{feature space} $A = \REAL^d$, and an integer $k\geq1$ that denotes the size of the partition (number of leaves), a function $s:A \to \REAL$ is a \emph{$k$-segmentation} if there is a partition $\B = \br{B_1,\cdots,B_k}$ of $A$ into $k$ disjoint axis-parallel hyper-rectangles (blocks), such that $|\br{s(b) \mid b \in B}|=1$ for every block $B \in \B$, i.e., $s$ assigns a unique value for all the entries in each of its $k$ rectangles; see Fig.~\ref{fig:defsPic}.
We define the union over all possible such $k$-segmentations by $\S_{(k,d)}$.
\end{definition}

We now define our loss function, and an optimal $k$-segmentation model over some set $A\subseteq\REAL^d$.
\begin{definition}[Loss function] \label{def:optSeg}
For a dataset $D = \{(x_1,y_1),\cdots,(x_N,y_N)\} \subseteq A\times\REAL$, an integer $k \geq 1$, and a $k$-segmentation $s \in \S_{(k,d)}$, we define the \emph{sum of squared error (SSE) loss}
\[
% \cost(D, s) = \sum_{(x,y) \in D} \left(s(x)-y\right)^2 
\cost(D,s) := \sum_{(x,y) \in D} \left(s(x)-y\right)^2
\]
as the loss of fitting $s$ to $D$.
A $k$-segmentation $s^*$ is an \emph{optimal $k$-segmentation} of $D$ if it minimizes $\cost(D, s)$ over every $k$-segmentation $s\in\S_{(k,d)}$ i.e., $s^*\in \argmin_{s\in \S_{(k,d)}} \cost(D, s)$. The optimal SSE loss is denoted by $\OPT_k(D) := \cost(D, s^*)$.
\end{definition}
% For example, if $k=1$, then $s$ may be the zero function and its corresponding approximation error is $\sum_{(x,y) \in D} y^2$. 
For example, the optimal $1$-segmentation $s^*$ of $D$ is the constant function $s^* \equiv \frac{1}{|D|} \sum_{(x,y) \in D} y$ since the mean of a set of numbers minimizes the sum of squared distances to the elements of the set. Also, $\OPT_{|D|}(D)=0$ for every dataset $D$.%If $k \geq |D|$ then the function $s(x) = y$ for every $(x,y)\in D$ would yield $0$ approximation error. The other cases where $1<k<|D|$ are less trivial, and are the motivation for this paper. 
% Mainly for simplicity, in this paper we focus on trees that aim to minimize the above mean squared error in the training data, i.e., the classic risk minimization model~\cite{vapnik1992principles}. We also assume the features are real numbers. See Future Work Section. 

We are now ready to formally define a coreset for the $k$-segmentation problem (and $k$-decision trees of at most $k$ leaves, in particular).
% \begin{definition} [Coreset for $k$-segmentation.] \label{def:epsCoreset}
% Let $D = \br{(x_1,y_1),\cdots,(x_n,y_n)} \subseteq [0,\infty)^2\times\REAL$ be an input dataset. Put $k\in [nm]$ and $\eps\in (0,1)$. A $(k,\eps)$-coreset for $D$ is data structure $(C,u)$ where $C \subseteq [0,\infty)^2\times\REAL$ and $w:C \to [0,\infty)$ is a weights function, such that using $(C,w)$ we can evaluate the cost $\cost(D,s)$ of the original dataset $D$, up to a multiplicative error of $(1\pm\varepsilon)$, in time that depends only on $|C|$, for any $k$-segmentation query $s$.
% \end{definition}
\begin{definition} [Coreset] \label{def:epsCoreset}
Let $D = \br{(x_1,y_1),\cdots,(x_n,y_n)} \subseteq A\times\REAL$ be an input dataset. Let $k \geq 1$ be an integer and $\eps\in (0,1)$ be the desired approximation error. A $(k,\eps)$-coreset for $D$ is a data structure $(C,u)$ where $C \subseteq A\times\REAL$ is an ordered set, and $u:C \to [0,\infty)$ is called a \emph{weight function}, such that $(C,u)$ suffices to approximate the loss $\cost(D,s)$ of the original dataset $D$, up to a multiplicative factor of $1\pm\varepsilon$, in time that depends only on $|C|$ and $k$, for any $k$-segmentation $s$.
\end{definition}

\textbf{Practical usage. }As defined above and discussed in Section~\ref{sec:coresets}, a coreset approximates every model in our set of models $\S_{(k,d)}$. Hence, a coreset for decision trees is clearly also a coreset for forests with an appropriate tuning for $k$, since every tree in the forest is approximated independently by the coreset.
We expect that applying existing heuristics (not necessarily with provable guarantees) such as sklearn~\cite{pedregosa2011scikit} or LightGBM~\cite{ke2017lightgbm} on the coreset, would yield similar results compared to the original data. Indeed, our experimental results in Section~\ref{sec:ER} validate those claims.

\textbf{No coreset for general datasets. }Unfortunately, even for the case of $k=4$ and $A \subseteq \REAL$, i.e., when the input is simply a one dimensional dataset $D = \br{(x_1,y_1),\cdots,(x_n,y_n)}$ where $x_1,\cdots,x_n$ are real numbers, and the labels $y_1,\cdots,y_n\in\br{0,1}$ have only binary values, it is easy to construct datasets which have provably no $k$-segmentation (or even $k$-tree) coreset of size smaller than $n$; see e.g.~\cite{rosman2014coresets}. 
% Since a $k$-segmentation is equivalent to a decision tree with $k$ leaves in $\REAL$, 
Hence, there is no non-trivial decision tree coreset for general datasets of $n$ vectors in any dimension. 
However, as we prove in the rest of the paper, a coreset does exist for datasets where the input is a matrix, i.e., a discrite signal where every coordinate in the domain is assigned a label (value), rather than a random set of $n$ vectors.

\textbf{The first coreset for $n\times m$-signals. }
To overcome the above problem, while still obtaining a small coreset, we assume a discretization of the dataset so that every coordinate has a label. We also assume, mainly for simplicity and lack of space, 
% that the dimension is $d=2$. Therefore, in what follows we assume that 
% that the input dataset assigns a real label for every entry in the feature space $A=[n]\times [m] \subseteq \REAL^2$. 
that the input feature space is $A=[n]\times [m] \subseteq \REAL^2$.
That is, the input can be represented by an $n\times m$ matrix. The output coreset may contain fraction of entries, as in Fig.~\ref{fig:blockCoreset}, which is called an \emph{$n\times m$ signal}; see Section~\ref{sec:prelim}. Our assumption on the input data seems to be the weakest assumption that can enable us to have a provably small coreset for any input. Furthermore, it seems natural for e.g. images, matrices, or any input data from sensors (such as GPS) that has a value in every cell or continuous in some other sense.

\textbf{Previous work. } The prior works~\cite{rosman2014coresets, feldman2012single, volkov2015coresets}, which only handle the case of segmenting a $1$-dimensional signal, use relaxations similar to our relaxation above to obtain a coreset of size $O(k/\varepsilon^2)$.
However, our results apply easily for the case of vectors ($1$-dimensional signals) as in~\cite{rosman2014coresets} and generalize for tensors if $d\geq 3$.
We also give further applications, and provide extensive experiments with popular state of the art software.

A special case for $d=2$ includes image compression, where quadtrees are usually used in e.g. MPEG4 to replace the image by smooth blocks of different sizes~\cite{shusterman1994image}, or for completion of missing values~\cite{twala2009empirical}
Using dynamic programming, it is easy to compute the optimal tree of a 2D-signal $D$ in $O(k^2n^5)$ time~\cite{bellman1966dynamic}, which is impractical even for small datasets, \emph{unless applied on a small coreset of $D$}. However we do not know of any such coreset construction, for $d \geq 2$, with provable guarantees on its size.
%, nor a provable $O(1)$ approximation that takes time less than $O(n^4)$. These facts were 
To this end, the following questions are the motivation for this paper: 
\textbf{(i): }Is there a small coreset for any $n\times m$ signal (e.g. of sub-linear size)? 
\textbf{(ii): }If so, can it be computed efficiently?
\textbf{(iii): }Can it be used on real-world datasets to boost the performance of existing random forest implementations?

\textbf{Extensions. }
For simplicity, we focus on the classic sum of squared distances (SSE) or the risk minimization model~\cite{vapnik1992principles}. However, our suggested techniques mainly assume that a coreset for the case $k=1$ is known, which is trivial for SSE, but exists for many other loss functions e.g., non-squared distances; see Section~\ref{sec:conclusion}.

% In practice, researchers usually use a forest, which is essentially a set of trees. We focus only on a single tree, since our coreset can be applied independently on each tree in the forest. Here we assume that if our coreset approximates the original input for each of those trees, the final result of the forest is also approximated which is true for common decision functions (e.g. majority vote).

%%%%%%%%%%%%%%%%%%%%%%%%%%%%%%%%%%%%%%%%%%%
%%%%%%%%%%%%%%%%%%%%%%%%%%%%%%%%%%%%%%%%%%%

\subsection{Our Contribution} \label{sec:contrib}
For any given error parameter $\eps\in(0,1)$, and an integer $k\geq 1$, this paper answers affirmatively the above three questions. More formally, in this paper we provide:\\
\textbf{(i): }A proof that \emph{every} $n\times m$ signal $D$ has a $(k,\eps)$-coreset $(C,u)$ of size $|C|$ polynomial in $k\log(nm)/\varepsilon$. To our knowledge, this is the first coreset for decision trees whose size is smaller than the input; see Theorem~\ref{theorem:coreset}. Due to lack of space, our full proofs are given in the appendix.\\
\textbf{(ii): }A novel coreset construction algorithm that outputs such a coreset $(C,u)$ with the above guarantees, for every given input signal $D$. Its running time is $O(nmk)$, i.e., linear in the input size $|D|$. Unlike common coreset constructions, our algorithm is deterministic; see Algorithm~\ref{Coreset}.\\
\textbf{(iii): }Experimental results that apply modern solvers, such as the sklearn and LightGBM libraries, on this coreset for real-world public datasets. We measured the trade-off between the empirical accuracy of the resulting forests, and the coreset size; see Section~\ref{sec:ER}.\\
\textbf{(iv): }AutoML for decision trees: since the suggested coreset approximates \emph{every} tree of \emph{at most} $k$ leaves, we may use the same coreset for hyperparameter tuning. We demonstrate this in Section~\ref{sec:ER}, by calibrating the parameter $k$ using only the coreset, as compared to using the original (big) data.\\
% and compare to running a similar tuning for the original (big) data.\\
\textbf{(v): }Open source code for our algorithms~\cite{opencode}. We expect that it will be used by both the academic and industry communities, not only to improve the running time of existing projects, but also to extend the algorithms and experimental results to other libraries and cost functions; see Section~\ref{sec:conclusion}.

\subsection{Novel technique: partition trees meet decision trees} \label{sec:novelty}
In a seminal paper~\cite{10.1145/10515.10522} during the 80's of the previous century,  Haussler and Welz introduced the importance of VC-dimension by Vapnik–Chervonenkis~\cite{vapnik2015uniform}.
Their main application was partition trees for answering range queries.

Informally, a \emph{partition tree} of a given set of points on the plane is the result of computing recursively a \emph{simplicial partition} which is defined as follows.
For a set $D$ of $N$ points on the plane, a \emph{$(1/\eps)$-simplicial partition} is the partition of $D$ into $O(1/\eps)$ subsets, such that: (i) each subset has at most $2 \eps N$ points, and (ii) Every line in the plane intersects the convex hull of at most $\sqrt{1/\eps}$ sets.
Answering range queries of the form ``how many points in $D$ are in a given rectangular'' in sub-linear time, using partition trees, is straightforward:
We can sum in $O(1/\eps)$ time the number of points in the subsets of the above simplicial partition that are not intersected by the query rectangular. We then continue recursively to count the points on each of the $\sqrt{1/\eps}$ intersected sets.

In other words, the main idea behind the above work is to partition the input into a (relatively small) number of subsets, each containing a fraction of the input, such that each query (in this case, a rectangular shape) might intersect only a small fraction of those subsets. Such a partition is termed a simplicial partition. The number of points contained in non-intersected subsets can be easily computed, while the sum of points in intersected subsets require a more involved solution. The novelty in that work is how to achieve such a partition of the input.

Our paper closes a loop in the sense that it forges links between decision trees in machine learning -- to partition trees from computational geometry. We aim to generalize the above technique from covering problems to regression and classification problems. This is by devising an algorithm which achieves the above requirements, but where the query is a decision tree (and not a rectangular shape), and the cost function is the sum of squared distances to the query and not the number of points.

More precisely, We partition the input dataset $D$ into a relatively small number of subsets, such that every possible decision tree (query) intersects at most few of these subsets. We then independently compress every subset via another novel algorithm such that the cost (sum of squared distances, in this case) of points contained in non-intersected subsets can be easily and accurately estimated, while the cost of points in intersected subsets can be provably approximated via a more involved calculation.
This is very different from existing coreset techniques that are sampling-based~\cite{langberg2010universal}, or utilize convex optimization greedy algorithms~\cite{clarkson2010coresets}.
Our main challenge was to define and design such a ``simplicial partition for sum of squared distances'', and the coreset to be computed for each subset in this partition.

\subsection{Preliminaries} \label{sec:prelim}
In this section we define the notation that will be used in the next sections.

Let $n\geq 1$ and denote $[n]=\br{1,\cdots,n}$.
An $n$-signal is a set $\{(x,f(x))\mid x\in [n]\}$ that is defined by a function $f:[n]\to\REAL$ (known as the graph of $f$). For an additional integer $m\geq1$, an $n\times m$ signal $D=\{(x,g(x))\mid x\in[n]\times [m]\}$ is the set that corresponds to a function $g:[n]\times [m]\to \REAL$. That is, $D$ represents an $n\times m$ real matrix whose size is $|D|=N=nm$.
For integers $i_1,i_2,j_1,j_2$ such that $1\leq i_1\leq i_2\leq n$ and $1\leq j_1\leq j_2\leq m$, an $n\times m$ sub-signal is the set
$B=\big\{(x,g(x))\mid x\in \{i_1,\cdots,i_2\}\times \{j_1,\cdots,j_2\}\big\} \subseteq D$; see Fig.~\ref{fig:defsPic}.
A sub-signal $B$ is called a \emph{row} (respectively, \emph{column}) if $i_1=i_2$ (respectively, $j_1=j_2$).
For a sub-signal $B$, we denote by ${B}^T = \br{((j,i),y) \mid ((i,j),y)\in B}$ the \emph{transposed sub-signal}. 
A $k$-segmentation $s$ is said to \emph{intersect} an $n\times m$ sub-signal $B$ if $s$ assigns at least two distinct values to the entries of $B$, i.e., $\left|\br{s(x) \mid (x,y) \in B}\right| \geq 2$. Furthermore, by definition, a $k$-segmentation $s$ induces a partition of an $n\times m$ sub-signal $B$ into at most $k$ $n\times m$ sub-signals. 
For two functions $f,g:\REAL \to \REAL$ we use the big $O$ notation $f(x) \in O(g(x))$, thinking of $O(g(x))$ as the class of all functions $h(x)$ such that $|h(x)| \leq c|g(x)|$ for every $x>x_0$, for some constants $c$ and $x_0$.
Lastly, we denote $\S_k:= \S_{(k,2)}$ for brevity.

\textbf{Paper organization.} Section~\ref{sec:bicretiria} provides a rough approximation to the $k$-segmentation problem. Section~\ref{sec:balancedPartition} provides an algorithm for computing a simplicial partition for the $k$-segmentation problem. Each region in this partition will be then compressed individually in Section~\ref{sec:coreset} to obtain our desired coreset.
Experimental results and discussions are given in Section~\ref{sec:ER}, and a conclusion in Section~\ref{sec:conclusion}. 

\section{Bi-criteria Approximation} \label{sec:bicretiria}
A coreset construction usually requires some rough approximation to the optimal solution as its input.
Unfortunately, we do not know how to \emph{efficiently} compute even a constant factor approximation to the optimal $k$-segmentation problem in Definition~\ref{def:optSeg}, as explained in Section~\ref{sec:intro}.
Instead, we provide an $(\alpha,\beta)_k$ or bi-criteria approximation~\cite{feldman2011unified}, where the approximation is with respect to a pair of parameters: the number of segments in the partition may be up to $\beta k$ instead of $k$, and the loss may be up to $\alpha \cdot \OPT_k(D)$ instead of $\OPT_k(D)$.%; 

\begin{definition} [$(\alpha,\beta)_k$-approximation.] \label{def:abApprox}
Let $D$ be an $n\times m$ sub-signal, $k \geq 1$ be an integer and let $\alpha, \beta > 1$.
A function $s:[n]\times[m] \to \REAL$ is an $(\alpha,\beta)_k$-approximation of $D$, if $s$ is a $\beta k$-segmentation whose fitting loss to $D$ is at most $\alpha$ times the loss of the optimal $k$-segmentation of $D$, i.e., $s \in \S_{(\beta k)}$ and
$\cost(D, s) = \sum_{(x,y) \in D} \left(s(x)-y\right)^2 \leq \alpha \cdot \OPT_k(D).$
\end{definition}

We now describe an algorithm that computes such an approximation, in time only linear in the input's size $|D|=nm$. The following lemma gives the formal statement. A suggested implementation for the algorithm is given in the appendix, as well as the full proof of the lemma; see Section~\ref{sec:bicretiria_appendix}.
\begin{lemma} \label{lem:ab}
Let $D = \br{(x_1,y_1),\cdots,(x_{nm},y_{nm})}$ be an $n\times m$ sub-signal and $k \geq 1$ be an integer. Then, there is an algorithm that can compute, in $O(knm)$ time, an $(\alpha,\beta)_k$-approximation for $D$, where $\alpha \in O(k\log (nm))$ and $\beta \in k^{O(1)} \log^2{(nm)}$.%Algorithm~\ref{Alg:bicriteria}.
\end{lemma}

\textbf{Overview of the bicriteria algorithm from Lemma~\ref{lem:ab}: }
% The algorithm aims to compute an $(\alpha,\beta)_k$-approximation for an input set $D$. 
The algorithm is iterative and works as follows. At the $i$th iteration, we find a collection $\B_i$ of at most $t$ disjoint sub-signals in $D_i$ (where $D_0 = D$ is the input), for which: 
(i) $\sum_{B \in \B_i} \OPT_1(B) \leq \OPT_k(D_i) \leq \OPT_k(D)$, and 
(ii)  $\cup_{B \in \B_i} B$ has size $ \abs{\cup_{B \in \B_i} B} \geq |D_{i}|/c$ for some parameter $c$ that depends on $k$, i.e., those sub-signals contain at least a $1/c$ fraction of $D_i$. We then define $D_{i+1} = D_i \setminus \cup_{B \in \B_i} B$.
After repeating this for at most $\psi \in O(c\log (nm))$ iterations, we end up covering all entries of $D$ with sub-signals where the overall loss of the sub-signals in each iteration is at most $\OPT_k(D)$.
This defines a partition of $D$ into a collection of at most $t \psi$ disjoint sets $\B'$, which, in turn, define a set of at most $(t\psi)^2$ distinct sub-signals. The output is now simply the function $s$ that assigns, for every $B \in \B'$ and $b \in B$, the mean value $s(b) = \frac{1}{\abs{B}} \sum_{((i,j),y) \in B}y$ of $B$. See Pseudo-code in Algorithm~\ref{Alg:bicriteria} at the appendix.

\section{Balanced Partition}\label{sec:balancedPartition}

In this section we present Algorithm~\ref{PTpartitionAllDims}, which computes a partition similar to the simplicial partition described in Section~\ref{sec:novelty};
It computes, in $O(|D|)$ time, a partition $\B$ of the input $D$ that satisfies the following properties: (i) $|\B|$ depends on $k/\varepsilon$ but independent of $|D|$, (ii) the loss $\OPT_1(B)$ of every $B \in \B$ is small, and (iii) every $k$-segmentation $s$ intersects only few sub-signals $B \in \B$; see Definition~\ref{def:balancedPartition}, Fig.~\ref{fig:balancedPartition}, and Lemma~\ref{PTpartition2D}. A full proof is given at the appendix; see Section~\ref{sec:balancedPartition_appendix}.

\begin{figure}[h]
  \centering
  \includegraphics[width=\textwidth]{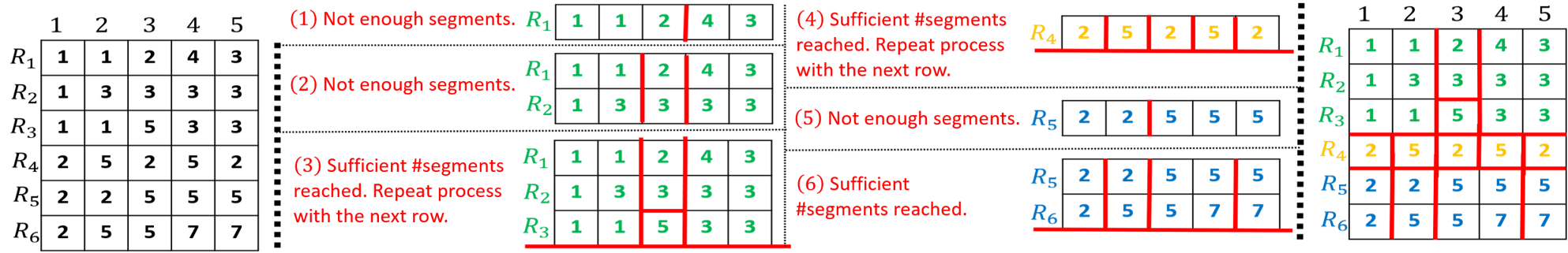}
  \caption{\textbf{(Left): }A $6\times 5$ signal $D$ consisting of $6$ rows $R_1,\cdots,R_6$. \textbf{(Middle): }A step by step illustration of the call to $\B := \partitionAlg(D,1/4,64)$ which partitions $D$ into $|\B|=13$ sub-signals (sub-matrices) as follows. (1) $\B' := \partitionAlgOneD(\br{R_1}, 4)$ (top green row). (2) $\B' := \partitionAlgOneD(R_1 \cup R_2, 4)$ (middle green matrix), and so on as long as the output contains at most $|\B'| \leq 1/\gamma = 4$ sub-signals (as in the bottom green matrix (3)). We then append $B'$ to the output $\B$, and repeat with the remaining $\br{R_4, R_5, R_6}$. (4) $\B' := \partitionAlgOneD(\br{R_4}, 4)$ which already returns $|\B'| = 5 > 1/\gamma$ signals (yellow matrix). We append them to $\B$ and repeat.
  \textbf{(Right): }The final partition $\B$, where $\OPT_1(B) \leq \gamma^2\sigma = 1/4^2 \cdot 64 = 4$ for every $B \in \B$.}
  \label{fig:balancedPartition}
\end{figure}

\begin{definition} [Balanced Partition] \label{def:balancedPartition}
Let $D$ be an $n\times m$ signal, $k\geq 1$ be an integer, and $c_1,c_2,c_3 > 0$. A $(c_1,c_2,c_3)_k$-balanced partition of $D$ is a partition $\B$ of $D$ such that: (i) $\B$ contains $|\B| \leq c_1$ $n\times m$ sub-signals, (ii) $\OPT_1(B) \leq c_2$ for every $B \in \B$, and (iii) every $k$-segmentation $\hat{s}$ intersects at most $c_3$ sub-signals $B \in \B$ (i.e., assigns more than one unique number to those sub-signals).
\end{definition}
% Note that the output of Lemma~\ref{lem:ab} satisfies properties (i) and (ii) in Definition~\ref{def:balancedPartition}.

\begin{lemma} \label{PTpartition2D}
Let $D$ be an $n\times m$ signal, $k\geq 1$ be an integer, $\varepsilon \in (0,1/4)$ be an error parameter, and
$s:[n]\times[m] \to \REAL$ be an $(\alpha,\beta)_k$-approximation of
$D$, where $\alpha,\beta > 1$. Define $\sigma := \frac{\cost(D,s)}{\alpha}$ and $\gamma :=
\frac{\varepsilon^2}{\beta k}$. 
Let $\B$ be an output of a call to $\partitionAlg(D,\gamma,\sigma)$; see Algorithm~\ref{PTpartitionAllDims}.
Then $\B$ an $\left(O\left(\frac{\alpha}{\gamma^2}\right), \gamma^2\sigma, O\left(\frac{k\alpha}{\gamma}\right)\right)_k$-balanced partition of $D$. Moreover, $\B$ can be computed in $O(nm)$ time.
\end{lemma}

\textbf{Overview of Algorithms~\ref{PTpartitionOneDim} and~\ref{PTpartitionAllDims}.}
Algorithm~\ref{PTpartitionAllDims} gets as input an $n\times m$-signal $D$ and two parameters $\sigma, \gamma$. 
Algorithm~\ref{PTpartitionAllDims} aims to compute a balanced partition of $D$; see Fig.~\ref{fig:balancedPartition}. In turn, it calls Algorithm \ref{PTpartitionOneDim}, which takes as input an $n\times m$ sub-signal $R$ that is defined
by several contiguous rows of the original dataset $D$, and a parameter $\sigma > 0$, and aims to compute a partition $\B$ of $R$. 
% Each of the $6$ steps in Fig.~\ref{fig:balancedPartition} is computed via a single call to Algorithm \ref{PTpartitionOneDim}.
To do so, Algorithm \ref{PTpartitionOneDim} partitions $R$ along the vertical dimension (e.g., into vertical slices), in a greedy fashion, such that for every $B \in \B$, $\OPT_1(B)$ is as large as possible, while still upper bounded by $\sigma$. 
% Assuming that the cost $c_1(R)$ is relatively small. 
This will ensure that the partition is into a relatively small number of slices. In the case where one of the sub-signals $B$ in this vertical partition of $R$ contains only one column, and already exceeds the maximum tolerance $\OPT_1(B) > \sigma$, we recursively apply Algorithm~\ref{PTpartitionOneDim} to $B^T$ in order to partition $B$ horizontally.
As long as the total number of slices returned by Algorithm~\ref{PTpartitionOneDim} is smaller than $1/\gamma$, Algorithm~\ref{PTpartitionAllDims} adds yet another row to the previous set of rows, and repeats the above process. 
% This is done as long as the number of sub-signals returned by Algorithm~\ref{PTpartitionOneDim} is smaller than $1/\gamma$.
At this point, the partition of the current horizontal slice (collection of rows) $R$ is final, and is added to the output partition of Algorithm~\ref{PTpartitionAllDims}. In turn, a new horizontal slice $R$ of just one row, the
first row of $D$ that is not included in the previous $R$, is initiated on which we again call
Algorithm \ref{PTpartitionOneDim}.

\begin{minipage}{0.51\textwidth}
\begin{algorithm}[H]
    \caption{\textsc{$\partitionAlgOneD(D, \sigma)$}}
    \label{PTpartitionOneDim}
    \SetKwInOut{Input}{Input}
	\SetKwInOut{Output}{Output}
    \Input{A parameter $\sigma > 0$ and an $n\times m$ signal $D = \br{(x_i,y_i)}_{i=1}^N$.}
    \Output{A partition $\B$ of $D$.}

    $\B := \emptyset$ and $c_{begin} := 1$
    
    \While{$c_{begin} \leq m$} 
    {
    
        $B := \br{\left((i,j),y\right)\in D \mid j = c_{begin}}$ \tcp{extract first column}
        
        \If{$\OPT_1(B) > \sigma$}
        {
              
                  $\B' := \partitionAlgOneD(B^T, \sigma)$ \label{line:1dRecursive} 
                  
                 $\B := \B \cup \br{{B'}^T\mid B' \in \B'}$ \label{line:addToOutput1D_2}
            %   }
              $c_{begin} := c_{begin}+1$
              
        } \Else 
        {
            $c_{end} := c_{begin}$
            
            \While{$\OPT_1(B) \leq \sigma \text{ and } c_{end} < m$ \label{line:condOptIsBig}}
            {
                $c_{end} := c_{end} + 1$ and $lastB := B$
                
                $B := \br{\left((i,j),y\right)\in D \mid i \in [c_{begin} ,c_{end}]}$
                
            }
            
            $\B := \B \cup \br{lastB}$ \label{line:addToOutput1D_3}
            
            $c_{begin} := c_{end}$
        }
    }
    
    \Return $\B$
\end{algorithm}
\end{minipage}
\hfill
\begin{minipage}{0.47\textwidth}
\begin{algorithm}[H]
    \caption{\textsc{$\partitionAlg(D,\gamma, \sigma)$}}
    \label{PTpartitionAllDims}
    \SetKwInOut{Input}{Input}
	\SetKwInOut{Output}{Output}
    \Input{An $n\times m$ signal $D$, a parameter $\gamma \in (0,1)$, and a lower bound $\sigma \in [0, \OPT_k(D)]$.}
    \Output{ A partition $\B$ of $D$; see Lemma~\ref{PTpartition2D}}

    $\B := \emptyset$ and $r_{begin}:=1$
    
    \While{$r_{begin} \leq n$} 
    {
        $R := \br{\left((i,j),y\right)\in D \mid i = r_{begin}}$ \tcp{extract first row}
        
        $\B' := \partitionAlgOneD(R, \gamma^2\sigma)$  
        
        $r_{end} := r_{begin}$
        
        $last\B' := \B'$
        
        \While{$|\B'| \leq 1/\gamma$ and $r_{end} < n$ } 
        {
        
             $r_{end} := r_{end} + 1$
            
            $last\B' := \B'$
            
            $S := \br{\left((i,j),y\right)\in D \mid i \in [r_{begin},r_{end}]}$ \tcp{extract a slice}
            
            $\B' = \partitionAlgOneD(S, \gamma^2\sigma)$ \label{line:partitionSlice}%\tcp{partition the slice}
        } 
        
        $\B := \B \cup last\B'$ \label{line:addToOutput}
        
        $r_{begin} := r_{end}$
    }

    \Return $\B$
\end{algorithm}
\end{minipage}

%%%%%%%%%%%%%%%%%%%%%%%%%%%%%%%%%%%%%%%%%%%%%%%%%%
\section{Coreset Construction} \label{sec:coreset}
In this section, we present our main algorithm (Algorithm~\ref{Coreset}), which outputs a $(k,\varepsilon)$-coreset for a given $n\times m$ signal $D$, the number of leaves $k\geq1$, and an approximation error $\eps\in (0,1)$.

\textbf{Overview of Algorithm~\ref{Coreset}:} The algorithm first utilizes the $(\alpha,\beta)_k$-approximation from Section~\ref{sec:bicretiria} to obtain a lower bound $\sigma \leq \OPT_k(D)$ for the optimal $k$-segmentation. It then computes, as described in Section~\ref{sec:balancedPartition}, a balanced partition $\B$ of $D$, where $\OPT_1(B)$ is small and depends on $\sigma$, for every $B \in \B$. 
Finally, it computes a small representation $(C_B,u_B)$ for every $B \in \B$, and returns the union of those representations. Each such pair $(C_B,u_B)$ satisfies: (i) $|C_B| = 4$, and (ii) has the same weighted sum of values, weighted sum of squared values, and sum of weights, as $B$, i.e., $ \sum_{(a,b) \in C_B} u_B((a,b))\cdot (b \mid b^2 \mid 1) = \sum_{(x,y) \in B}(y \mid y^2 \mid 1)$; see Fig.~\ref{fig:blockCoreset}. Such a representation can be computed using Caratheodory's theorem, as explained in Section~\ref{sec:cara} of the supplementary material.

\begin{figure}[h]
\centering
\includegraphics[width=\textwidth]{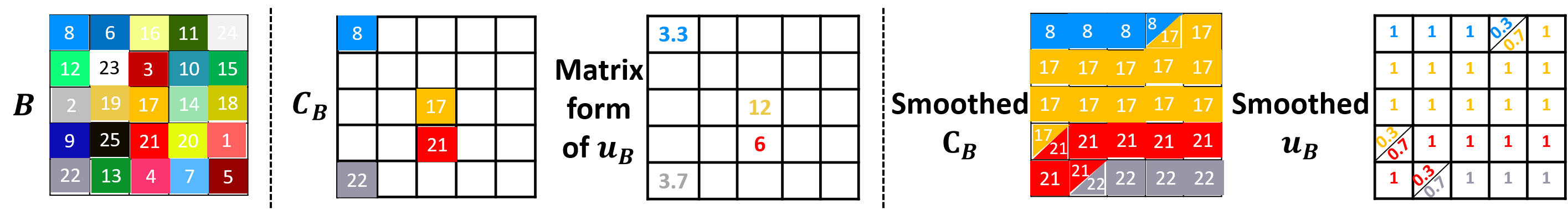}
\caption{\textbf{(Left): }A matrix representing of a $5\times 5$ sub-signal $B$ where $y$ is mapped into unique colors for every $(x,y) \in B$. \textbf{(Middle): }A representative (coreset) pair $(C_B,u_B)$ for $B$ where $C_B \subseteq B$ is a (small) subset and $u_B:C_B \to [0,\infty)$ is a weight function. That is, the pair $(C_B,u_B)$ satisfies $\sum_{(a,b)\in C_B}u_B((a,b)) \cdot (b \mid b^2 \mid 1) = \sum_{(x,y) \in B}(y \mid y^2 \mid 1)$. \textbf{(Right): }A duplication of the coreset points according to their weight. We call the resulting pair a ``smoothed'' version of $(C_B,u_B)$; see more details and formal definition in Section~\ref{sec:coresetConstruction_appendix} of the supplementary material.}
\label{fig:blockCoreset}
\end{figure}

\textbf{Some intuition behind Algorithm~\ref{Coreset}: }Consider some $k$-segmentation $s$. 
By the properties of the balanced partition $\B$ of $D$, only a small number of sub-signals $B \in \B$ are intersected by $s$, i.e., assigned at least $2$ distinct values. For every non-intersected sub-signal $B \in \B$, the loss $\cost(B,s)$ is accurately estimated by the (coreset) pair $(C_B,u_B)$. On the other hand, for every sub-signal $B \in \B$ which is intersected by $s$, by the guarantees of the representation $(C_B,u_B)$, the loss $\cost(B,s)$ will be approximated, using only $(C_B,u_B)$, up to some small error that depends on $\OPT_1(B)$. However, again by the properties of $\B$, we have that $\OPT_1(B)$ is small. Hence, using the union $(C,u)$ of the representations we can approximate $\cost(D,s)$ as required. Furthermore, combining that $|C_B| \in O(1)$ for every $B \in \B$ with the fact that $|\B|$ is small yields that $|C|$ is indeed small; see Theorem~\ref{theorem:coreset}.

\begin{algorithm}[h]
    \caption{\textsc{$\coresetAlg(D,k,\varepsilon)$}; see Theorem~\ref{theorem:coreset}}
    \label{Coreset}
    \SetKwInOut{Input}{Input}
	\SetKwInOut{Output}{Output}
    \Input{An $n\times m$ signal $D$, an integer $k \geq 1$, and an error parameter $\varepsilon \in (0,1/4)$.}
    \Output{A $(k,\varepsilon)$-coreset $(C, u)$ for $D$.}

    $s := $ an $(\alpha,\beta)_k$ approximation of $D$ for $\alpha \in O(k\log (nm))$ and $\beta \in k^{O(1)} \log^2{(nm)}$ \label{line:callBic} \texttt{;see Lemma~\ref{lem:ab} for suggested implementation.}

    $\gamma := \varepsilon^2/(\beta k)$, $\sigma := \frac{\cost(D, s)}{\alpha}$ and $C := \emptyset$ \label{line:consts}

    $\B := \partitionAlg(D, \gamma, \sigma)$ \label{line:callPartition} \tcp{see Algorithm~\ref{PTpartitionAllDims}.}

    \For{every set $B \in \B$ \label{line:everyBlock}} 
    {
        $(C_B, u_B) :=$ a $(1,0)$-coreset for $B$, (a zero error coreset for $k=1$), such that $C_B\subseteq B$, $|C_B| = 4$, and $ \sum_{(a,b) \in C_B} u_B((a,b))\cdot (b \mid b^2 \mid 1) = \sum_{(x,y) \in B}(y \mid y^2 \mid 1)$ 
    \label{line:compBlockCoreset} \texttt{this is done using Caratheodory's theorem; see Corollary~\ref{cor:cara} in the appendix.} 
    
    Replace each of the coordinates $a$ of the $4$ pairs $(a,b) \in C$ with one of the $4$ corner coordinates of the pairs in $B$ \label{Line:replaceCoords} \texttt{; see detailed explanation if the proof of Theorem~\ref{theorem:coreset}.}
    
        $C := C \cup C_B$ and $u((a,b)) := u_B((a,b))$ for every $(a,b) \in C_B$.
    }

    \Return $(C, u)$
\end{algorithm}

\begin{theorem}[Coreset] \label{theorem:coreset}
Let $D =\br{(x_1,y_1),\cdots,(x_{N},y_{N})}$ be an $n\times m$ signal i.e., $N := nm$. Let $k\geq 1$ be an integer (that corresponds to the number of leaves/rectangles), and $\varepsilon \in (0,1/4)$ be an error parameter. Let $(C,u)$ be the output of a call to $\coresetAlg(D,k,\varepsilon/\Delta)$ for a sufficiently large constant $\Delta\geq 1$; see Algorithm~\ref{Coreset}. Then, $(C,u)$ is a $(k,\varepsilon)$-coreset for $D$ of size $|C| \in \frac{(k\log(N))^{O(1)}}{\varepsilon^4}$; see Definition~\ref{def:epsCoreset}. Moreover, $(C, u)$ can be computed in $O(kN)$ time.
\end{theorem}

\textbf{Coreset size. }
While Theorem~\ref{theorem:coreset} gives a worst-case theoretical upper bound, this bound is too pessimistic in practice, as common in coreset papers~\cite{lucic2017training, jubran2020sets}.
This phenomenon is well known for coresets; see discussion e.g., in~\cite{feldman2020core,ros2020sampling}. The reasons might include: worst-case artificial examples vs. average behaviour on structured real-world data, noise removing/smoothing by coresets, the fact that in practice we run heuristics that output a local minima (and not optimal solutions with global minimum), non-tight analysis (especially when it comes to constants), etc.

Our experiments in Section~\ref{sec:ER} show that, empirically, the constructed coresets are significantly smaller: for $N := nm \sim 140,000$, $k=1000$, and $\varepsilon = 0.2$, Theorem~\ref{theorem:coreset} predicts, in the worst case, a coreset of size larger than the full dataset size $N$. However, such an $\varepsilon$ error is obtained with a coreset of size at most $1\%$ of the input; see Fig.~\ref{fig:ER}.

\section{Experimental Results} \label{sec:ER}
We implemented our coreset construction from Algorithm~\ref{Coreset} in Python $3.7$, and in this section we evaluate its empirical results, both on synthetic and real-world datasets. More results are placed in the supplementary material; see Section~\ref{sec:additionalER}. Open-source code can be found in~\cite{opencode}. The hardware used was a standard MSI Prestige 14 laptop with an Intel Core i7-10710U and 16GB of RAM.
Since our coreset construction algorithm does not compete with existing solvers, but improves them by reducing their input as a pre-processing step, we apply existing solvers as a black box on the small coreset returned by Algorithm~\ref{Coreset}. The results show that our coreset can boost, by up to x$10$ times, the running time and storage cost of common random forest implementations.

\textbf{Implementations for forests.}
We used the following common implementations: (i) the function \randomForest{} from the \texttt{sklearn.ensemble} package, and (ii) the function \lightGBM{} from the \texttt{lightGBM} package that implements a forest of gradient boosted trees. 
Both functions were used with their default hyperparameters, unless states otherwise.\\
\textbf{Data summarizations.}
We consider the following compression schemes:\\
\textbf{(i): }$\DTcoreset(D,k,\varepsilon)$ - The implementation based on Algorithm~\ref{Coreset}.
In all experiments we used a constant $k=2000$ for computing the coreset, regardless of the (larger) actual $k$ value in each test, since $k=2000$ was sufficient to obtain a sufficiently small empirical approximation error. Hence, the parameter $\varepsilon$ controls the trade-off between size and accuracy. \\
\textbf{(ii): }$\randomSample(D, \tau)$ - returns a uniform random sample of size $\tau$ from $D$. In all tests $\tau$ was set to the size of the coreset $\DTcoreset(D,k,\varepsilon)$ for fair comparison.\\
\textbf{Datasets.}
We used the following pair of datasets from the public UCI Machine Learning Repository~\cite{asuncion2007uci}, each of which was normalized to have zero mean and unit variance for every feature:\\
\textbf{(i): Air Quality Dataset} \cite{de2008field} - contains $n=9358$ instances and $m=15$ features.\\
\textbf{(ii) Gesture Phase Segmentation Dataset} \cite{madeo2013gesture} -  contains $n=9900$ instances and $m=18$ features.\\
% Both datasets were normalized to have zero mean and unit variance for every feature.
\textbf{The experiment.} The goal was to predict missing entries in every given dataset, by training random forests on the available data. The test set (missing values) consists of $30\%$ of the dataset, and was extracted from the input dataset matrix by randomly and uniformly choosing a sufficient number of $5\times 5$ patches in the input dataset, and defining them as missing values. The final loss of a trained forest is the sum of squared distances between the forest predictions for the missing values, and the ground truth values.
To tune the hyperparameter $k$, we randomly generate a set $\K$ of possible values for $k$ on a logarithmic scale. Then, we either: (i) apply the standard tuning (train the forest on the full data, for each value in $\K$, and pick the one with the smallest test set error), or (ii) compress the input (only once) into a small representative set, and then apply the standard tuning on the small, rather than the full, data. 
The experiment was repeated $10$ times. All the results are averaged over all $10$ tests; see Fig.~\ref{fig:ER}.

\begin{figure}[h]
\centering
\includegraphics[width=\textwidth]{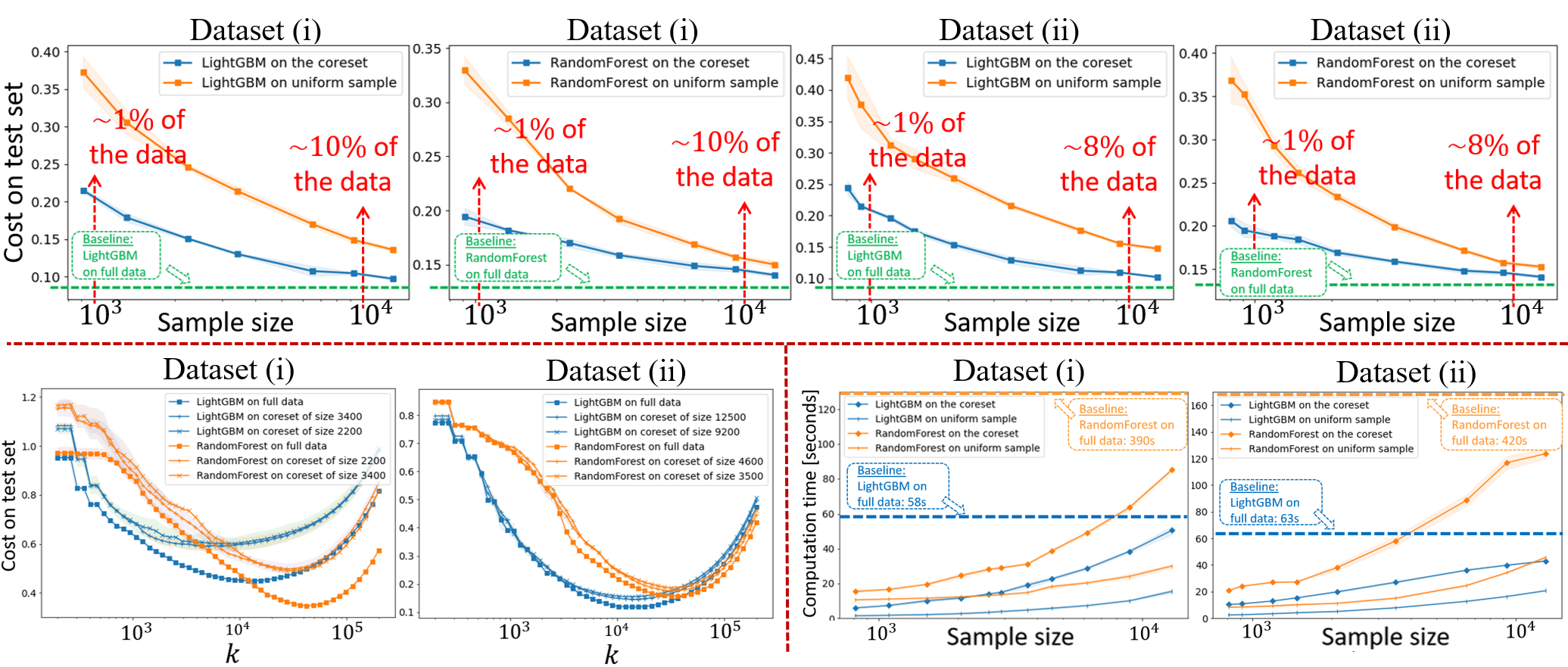}
\caption{\textbf{Experimental results.} \textbf{(Top): }The $X$-axis is the compression size. 
For every compression size $\gamma$, hyperparameter tuning is applied on both the coreset and the uniform sample (which are both of size $\gamma$). A random forest is then trained, on the full data, using those tuned parameters. The $Y$-axis presents the test set SSE loss of the trained forests. 
\textbf{(Bottom left): }Hyperparameter tuning. For every different value of $k$ ($X$-axis), a forest is trained using this parameter value either on the compression (of two different sizes) or on the full data. The $Y$-axis presents $\ell+ k/10^5$, where $\ell$ is the normal SSE loss of the trained forest on the test set.
\textbf{(Bottom right): }Time comparison. The $Y$-axis presents the total running time of both to compute the compression and to tune the parameter $k$ on the compression (out of $50$ different values). 
Note that the bottom right figures measure the total time to tune the parameter $k$ in the bottom left figures, but using many more compression sizes. The optimal obtained parameter was then used to train the random forest in the top figures.}
\label{fig:ER}
\end{figure}

\textbf{Discussion. }
While the size and accuracy of our coreset are independent of our exact implementation of Algorithm~\ref{Coreset}, the running time is heavily based on our naive implementation, as compared to the very efficient professional Python libraries.
This explains why most of the running time is still devoted to the coreset construction rather than the forest training. Nevertheless, even our simple implementation yielded improvements of up to x$10$ in both computational time and storage, for a relatively small accuracy drop of $0.03$ in the SSE. 
Tuning more than one hyperparameter will result in a bigger improvement. Furthermore, Fig.~\ref{fig:ER} empirically shows that tuning a hyperparameter on the coreset yields a loss curve very similar to the loss curve of tuning on the full data. Lastly, we observe that, in practice, our coresets have size much smaller than predicted in the pessimistic theory.

\section{Conclusions and Future Work} \label{sec:conclusion}

While coresets for $k$-trees do not exist in general, we provided an algorithm that computes such a coreset for every input \emph{$n\times m$ signal}. The coreset size depends polynomialy on $k\log(nm)/\varepsilon$ and can be computed in $O(nmk)$ time.
Our experimental results on real and synthetic datasets demonstrates how to apply existing forest implementations and tune their hyperparameters on our coreset to boost their running time and storage cost by up to x$10$.
In practice our coreset works very well also on non-signal datasets, probably since they have ``real-world'' properties that do not exist in the artificial worst-case example from Section~\ref{sec:firstCoreset}. An open problem is to define these properties. Moreover, while this paper focuses on the sum of squares distances loss, we expect that the results can be generalized to support other loss functions; see Section~\ref{sec:firstCoreset}. Lastly, supporting high-dimensional data (tensors), instead of matrices, is a straightforward generalization that can be achieved via minor modifications to our algorithms. Due to space limitation we also leave this to future work.

\bibliography{references_DT}
\bibliographystyle{plain}

\clearpage

\appendix

\section{Additional Experiments} \label{sec:additionalER}

In this section, we present some additional experiments conducted using our algorithm from Sections~\ref{sec:balancedPartition}-\ref{sec:coreset}. We give visual illustration both of our coreset itself and of the result of applying a very common decision tree implementation on the coreset, as compared to running the same function on the original (full) data; see Fig.~\ref{fig:coreset_blobs}-\ref{fig:coreset_moons}.

\begin{figure}
    \centering
    \includegraphics[scale=0.8]{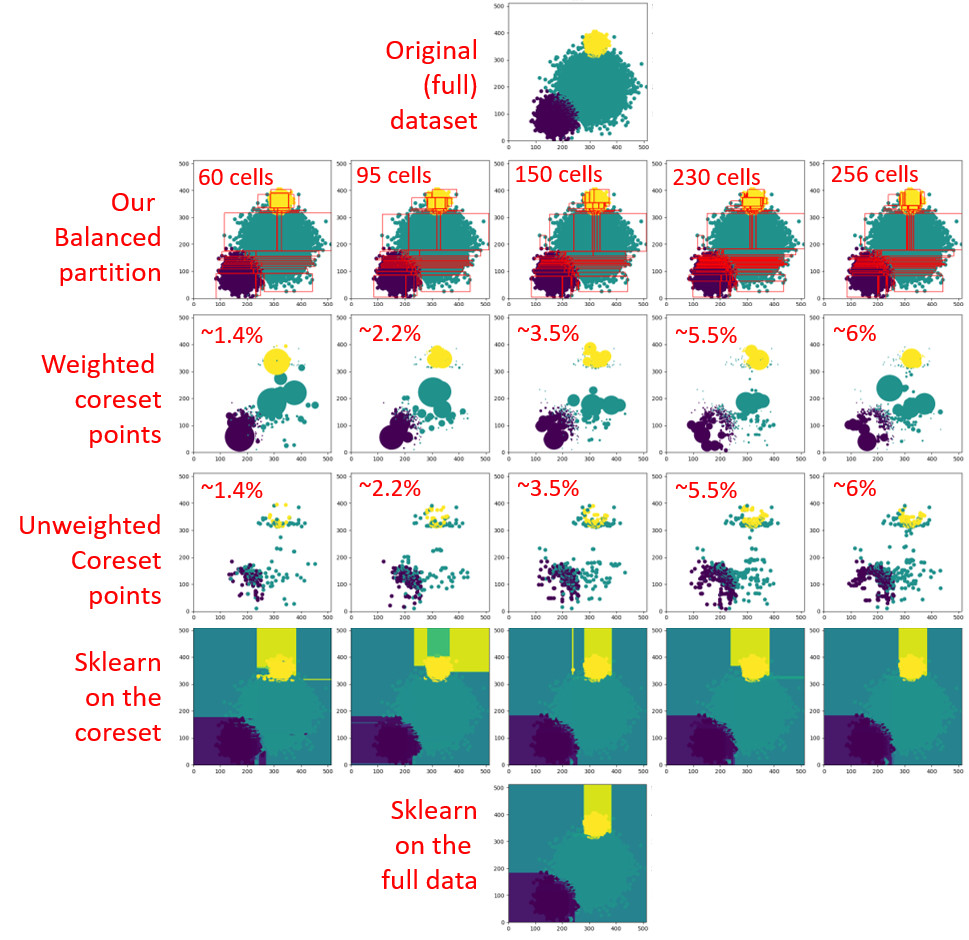}
    \caption{The blobs dataset. The dataset $D$ was generated using the function \texttt{sklearn.datasets.make\_blobs}, and contains $n=17,000$ points clustered into $3$ clusters (containing $8500, 5800$, and $2700$ points), each with a different label. A coreset $(C,u)$ was constructed using Algorithm~\ref{Coreset}. From top down, the rows illustrate: (i) The input dataset $D$, (ii) The balanced partition of $D$, including the number of sets in the partition (iii) The weighted coreset points. Each point $(x,y) \in C$ is plotted at location $x$, colored according to its label $y$, and its radius is proportional to its weight $u(x,y)$. The percentage of the coreset size relative to the full data is presented. (iv) The unweighted coreset points. Each point $(x,y)\in C$ is plotted at location $x$, colored according to $y$, and has a fixed radius. The percentage of the coreset size relative to the full data is presented. (v) The partition of the space via a decision tree computed using a call to \texttt{DecisionTreeRegressor} from the \texttt{sklearn.tree} package, where the input was the weighted coreset points only. Each region is assigned a color according to the label assigned to it by the computed tree. (vi) Similar to Row (v), but where the decision tree is trained on the full data. Algorithm~\ref{CoresetQuery} was used during training of the decision tree to evaluate the loss of each model.}
    \label{fig:coreset_blobs}
\end{figure}

\begin{figure}
    \centering
    \includegraphics[scale=0.8]{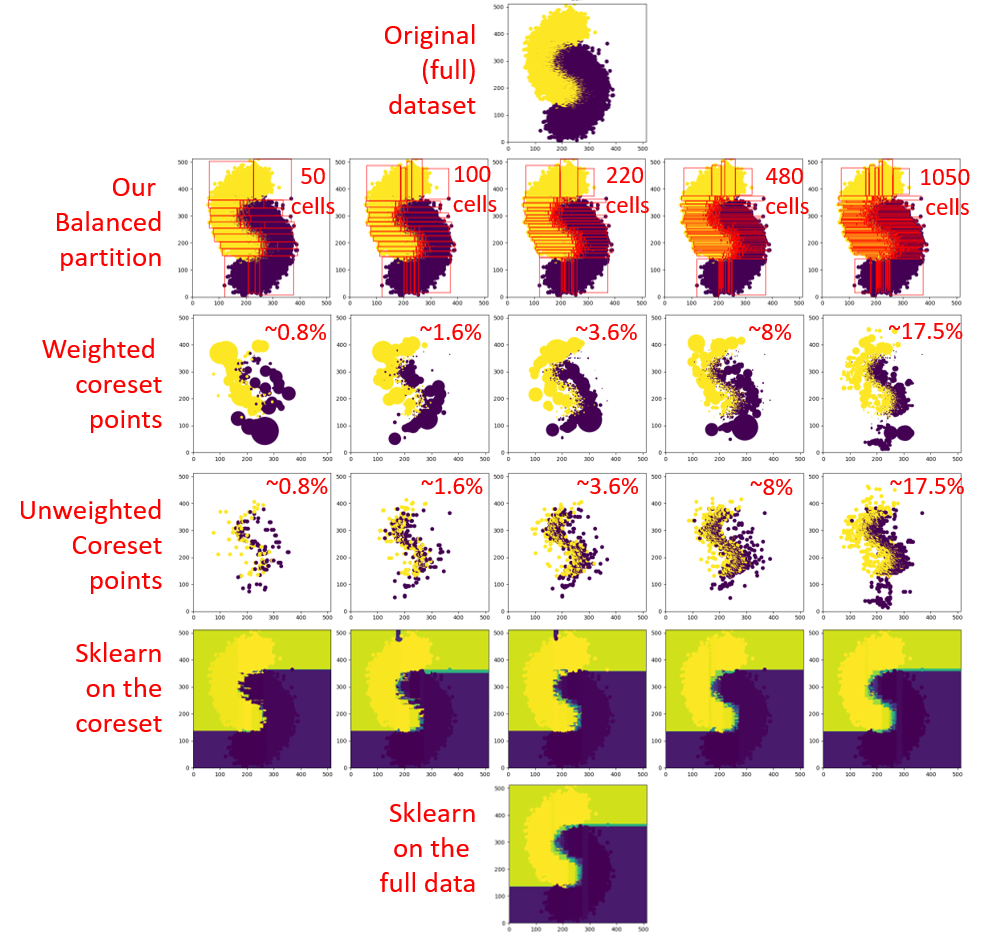}
    \caption{The moons dataset. The dataset was generated using the function \texttt{sklearn.datasets.make\_moons}. The dataset contains $n=24,000$ points spread across two interleaving half circles ($12,000$ points for each half circle), each with a different label. See caption of Fig.~\ref{fig:coreset_blobs} for a detailed explanation about the rows.}
    \label{fig:coreset_moons}
\end{figure}

\begin{figure}
    \centering
    \includegraphics[scale=0.8]{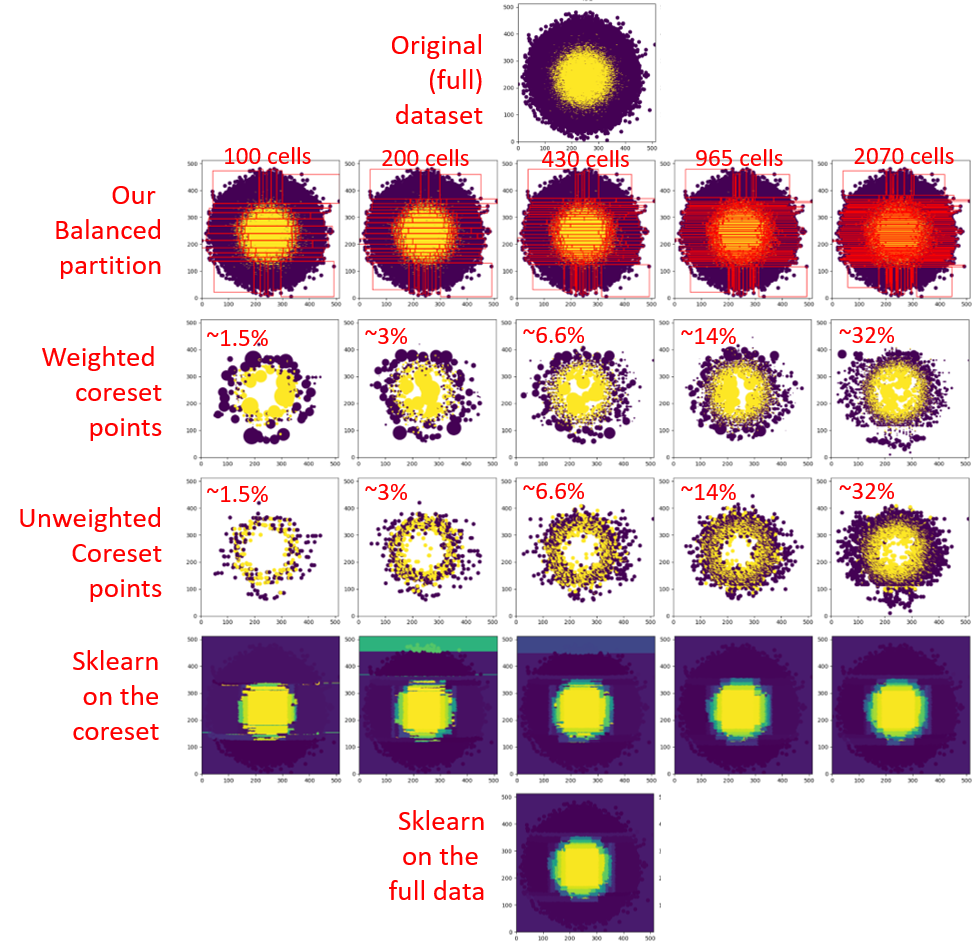}
    \caption{The circles dataset. The dataset was generated using the function \texttt{sklearn.datasets.make\_circles}. The dataset contains $n=26,000$ points spread across a big circle ($14,000$ points) and a small circle ($12,000$ points), each with a different label. See caption of Fig.~\ref{fig:coreset_blobs} for a detailed explanation about the rows.}
    \label{fig:coreset_circles}
\end{figure}

\paragraph{Discussion. } 
\textbf{Visual representation. }As seen in Fig.~\ref{fig:coreset_blobs}-\ref{fig:coreset_moons}, the balanced partition in the second row partitions the input data into multiple subsets, where, as expected, flat and relatively smooth regions are partitioned into a smaller number of large cells, while more complex regions are partitioned into a larger number of finer cells. This is expected since the balanced partition insures a small variance inside each cell. 

Furthermore, as seen in the third row of the above figures, the weighted coreset contains a small number of ``large'' circles (points with large weight) in the flat and relatively smooth regions, while it contains large number of ``small'' circles (points with small weight) in the more complex regions of the input.

\textbf{Accuracy. }As seen in the last two rows of Fig.~\ref{fig:coreset_blobs}-\ref{fig:coreset_moons}, the decision tree trained only on the coreset points resembles the decision tree trained on the full data, even for coresets of size only $6\%$, $8\%$, and $14\%$ of the input data, as seen in Fig.~\ref{fig:coreset_blobs},\ref{fig:coreset_moons}, and~\ref{fig:coreset_circles} respectively. This implies a x$10$ faster training time of a decision tree (or, similarly, a forest) on a given coreset, compared to training it on the full data, with almost no compromises to the accuracy.

The difference in the coreset size required in order to accurately represent the full data depends on the complexity of the input dataset. Indeed, the dataset in Fig.~\ref{fig:coreset_blobs} is a much simpler dataset for a decision tree to classify, compared to the dataset in Fig.~\ref{fig:coreset_circles}. Hence, the coreset sizes required in Fig.~\ref{fig:coreset_blobs} are smaller than the ones in Fig.~\ref{fig:coreset_circles}.

\section{Bi-criteria Approximation} \label{sec:bicretiria_appendix}
\paragraph{Notations. }A sub-signal $B$ is said to be \emph{horizontally intersected} by a $k$-segmentation function $s$ if there are $((i_1, j_1),y_1), ((i_2,j_2),y_2) \in B$ where $i_1 \neq i_2$ such that $s(i_1, j_1) \neq s(i_2,j_2)$. Similarly, a block $B$ of $D$ is said to be \emph{vertically intersected} by $s$ if there are $((i_1, j_1),y_1), ((i_2,j_2),y_2) \in B$ where $j_1 \neq j_2$ such that $s(i_1, j_1) \neq s(i_2,j_2)$. 
$B$ is said to be \emph{intersected} by $s$ if $B$ is either horizontally or vertically intersected, i.e., $\left|\br{s(x) \mid (x,y) \in B}\right| > 1$.
A set of sub-signals $\B$ is said to be horizontally (vertically) intersected by $s$ if it contains a sub-signal $B \in \B$ that is horizontally (vertically) intersected by $s$.%; see Fig.~\ref{fig:intersections} for an illustration.

Also, we might abuse notation and denote by signal (sub-signal) an $n\times m$ signal (sub-signal) and by $k$-segmentation an $n\times m$ $k$-segmentation. 

In this section we give a constructive proof for Lemma~\ref{lem:ab}. A suggested implementation for this constructive proof is given in Algorithm~\ref{Alg:bicriteria}.

\begin{algorithm*}[h]
    \caption{\textsc{$\bicriteriaAlg(D,k)$}; Lemma~\ref{lem:ab}}
    \label{Alg:bicriteria}
    \SetKwInOut{Input}{Input}
	\SetKwInOut{Output}{Output}
    \Input{An $n\times m$ sub-signal $D = \br{(x_i,y_i)}_{i=1}^N$ and an integer $k \geq 1$.}
    \Output{An $(\alpha,\beta)_k$-approximation for $D$.}
    
    $\B := \emptyset$
    
    $\nu, \gamma :=$ sufficiently large constants \tcp{see proof of Lemma~\ref{lem:ab}}
    \While{$|D| > k\log{N}$}
    {
        \If{$D$ contains a row $R$ with $|R| \geq \frac{|D|}{\nu k}$} 
        {
            Partition $[m]$ into $t'=\gamma k$ intervals $[m] = \cup_{j=1}^{t'} I_j$ such that every $j\in [t']$, the size of each corresponding sub-signal $R_j = \br{((x_1,x_2),y)\in R \mid x_2 \in I_j}$ is $\abs{R_j} \in \br{\frac{|R|}{t'}-1, \frac{|R|}{t'}+1}$. 
            \tcp{e.g., by a greedy pass over $[m]$.}
            
            $\B :=$ the set of $t'-2k$ signals $R_j$ with the smallest $\OPT_1(R_j)$.
        } 
        \Else 
        {
            Partition $[n]$ into $\psi$ intervals $[n] = \cup_{j=1}^{\psi} I_j$ such that for every $j\in [\psi]$, the size of each corresponding sub-signal $D_j = \br{((x_1,x_2),y)\in D \mid x_1 \in I_j}$ is $\frac{|D|}{\nu k} \leq \abs{D_j} \leq \frac{2|D|}{\nu k}$. 
            \tcp{e.g., by a greedy pass over $[n]$.}
            
            \If{at least $\psi/2$ of the sub-signals $D_j$ do not contain a column $col$ of size $|col| \geq \frac{|D_j|}{2(\nu k)^2}$}
            {
                Vertically partition each of the (at least $\psi/2$) sub-signals $D_j$ into $\psi_j$ sub-signals, each such sub-signal $B$ of size $\frac{|D_j|}{2(\nu k)^2} \leq |B| \leq \frac{|D_j|}{(\nu k)^2}$, and let $\B'$ contain the union of all those sub-signals. \texttt{e.g., via a greedy algorithm.}
                
                $\B :=$ the set of $|\B'|-4\nu^2k^3 -2k\psi$ signals $B \in |\B'|$ with the smallest $\OPT_1(B)$.
            } 
            \Else 
            {
            $\B := \br{C \mid C \text{ is a column of }D_j, |C| \geq \frac{|D_j|}{2 (\nu k)^2} \text{ and }j \in [\psi]}$
            % the union of all the columns $C \subseteq D_j$ for which $|C| \in \Omega(D_j/k^2)$, for every $j\in [\psi]$.
            } 
        } 
        $D := D \setminus \cup_{B \in \B'} B$ and $\B := \B \cup \B$
    }
    $s(b) := 1/|B| \sum_{(x,y) \in B}y$ for every $b \in B$ and $B \in \B$.
    
    \Return $s$
\end{algorithm*}

We first prove a small technical observation (see Observation~\ref{claim:var}), and then we prove Lemma~\ref{lem:partition}, which will be used throughout the proof of Lemma~\ref{lem:ab}.

\begin{observation} \label{claim:var}
Let $A$ and $B$ be two $n\times m$ sub-signals. Then it holds that
\[
\OPT_1(A \cup B) \geq \OPT_1(A) + \OPT_1(B).
\]
\end{observation}
\begin{proof}
Let $C = A\cup B$.
Let $\mu = \frac{1}{|C|} \sum_{(x,y) \in C} y$ be the weighted mean of $A\cup B$.
By Definition of $\OPT$ we have that
\[
\begin{split}
\OPT_1(A \cup B) & = \sum_{(x,y) \in C} (y-\mu)^2 \\
& = \sum_{(x,y) \in A}(y-\mu)^2 + \sum_{(x,y) \in B}(y-\mu)^2\\
& \geq \OPT_1(A) + \OPT_1(B),
\end{split}
\]
where the first derivation holds since the mean of a points minimizes the sum of squared distances to those points, and the last derivation is by the definition of $\OPT_1$.
\end{proof}

\begin{lemma} \label{lem:partition}
Let $D = \br{(x_1,y_1),\cdots,(x_N,y_N)}$ be an $n\times m$ sub-signal and let $k\geq 1$ be an integer. Then, in $O(N)$ time we can find a set $\B$ of $|\B| = t \in O(k^3)$ mutually disjoint blocks with respect to $D$, for which 
\begin{enumerate}
\item $\sum_{B \in \B} \OPT_1(B) \leq \OPT_k(D)$.
% \item $\abs{\cup_{B \in \mathcal{B}}B} \geq \frac{|A|}{8 \nu k}$.
\item $\abs{\cup_{B \in \B}B} \in \Omega\left(\frac{N}{k}\right)$.
\end{enumerate}
\end{lemma}
\ifproofs
\begin{proof}
Let $\nu > 50$ be an arbitrary parameter and let $\gamma\geq 8$ be a parameter that will be defined later.
We will prove Lemma~\ref{lem:partition} for $t \leq 2\nu^3k^3$ and for $\abs{\cup_{B \in \B}B} \geq \frac{N}{8\nu k}$.

We start with the simple $1$-dimensional case, namely -- we assume that $m=1$. In this case, we just partition $[n]$ into
$t'=\gamma k$ consecutive intervals $[n] = \cup_1^{t'} E_j$, such that each corresponding sub-signal $D_j = \br{((a,b),y) \in D \mid a \in E_j}$, $j \in [t']$ of $D$ has equal share of elements in $D$ (up to $\pm 1$), i.e., $|D_j| \in \br{|D|/t'-1, |D|, |D|/t'+1}$ . 
% Here $\gamma\geq 8$ is parameter that is left free at this point. 
This can be done
by moving point by point along the elements of $D$, which are assumed to be sorted in ascending order according to $a$ for every $((a,b),y) \in D$, and defining a new interval at the first moment the current interval contains more than $|D|/t'$ elements of $D$ or just at the prior point. We then compute for each $D_j, ~j\in [t']$ the optimal $\OPT_1(D_j),$ and return as output the set $\B$ containing the $t=t'-2k$ sub-signals $D_j$ with the smallest loss $\OPT_1(D_j)$, among all the $t'$ sub-signals $\br{D_1,\cdots,D_{t'}}$.

To see what guarantees we get, note that the computation takes
$O(nm) = O(N)$ time to sort the elements $((a,b),y)$ of $D$ according to $a$, since $a \in [nm]$ is a bounded integer. Afterwards, the above greedy partition also takes $O(nm) = O(N)$ time. Furthermore,
\[
\abs{\cup_{B \in \B} B} \geq (\gamma-2)k \cdot \frac{|D|}{\gamma k} \geq \frac{\gamma-2}{\gamma} |D| \geq \frac{N}{8\nu k}.
\]
For $\gamma \geq 8$ this proves Property (ii) above.

Finally, any $n\times 1$ $k$-segmentation function intersects at most $2k$ sub-signals from $\{D_j\}_{j=1}^{t'}$ (i.e., at most $2k$ sub-signals are assigned more than $1$ distinct value via the $k$-segmentation function). Hence, the optimal $k$-segmentation $s^*$ of $D$ intersects at most $2k$ of those sub-signals as well.
This implies that at least $t' -2k \geq (\gamma-2)k$ of the intervals $\{D_j\}_{j=1}^{t'}$ which are assigned $1$ distinct value $\left|\br{s(x) \mid (x,y) \in D_j}\right| = 1$. Hence, by Observation~\ref{claim:var}, we conclude that $\OPT_k(D) \geq \sum_{B \in \B} \OPT_1(D)$ which verifies the 1st item above.

We remark that for the $1$-dimensional case we can do much better
(there is  an overall $(1+\eps)$-approximation of the $k$-segmentation using logarithmic number of blocks), but this will not be used here. Another remark is that we have ignored the $\pm 1$ slack in the sizes above, making the actual part of $D$ that is removed at least $\frac{\gamma -2}{\gamma}|D| - t'$. This is
insignificant in the $1$-dimensional case above, as for $t= O(1)$ and for $|D| \geq \log n$ this would be an insignificant fraction, while for smaller $D$, we can just use
single point sub-signals. 
% We make this final remark, as while for the $1$-dim case the inaccuracy above is insignificant, for the $2$-dim case the corresponding additive slack, that comes from the fact that $D$ induces a discrete rather smooth measure, might be significant.

\noindent
{\bf The 2-dim case }: 
%For any block $A'$ in $A$ we denote by $E(A') = \{(a,i) \in A' \cap E \}$ the set of $E$-points in $A'$. 
Consider a row $R$ of $D$, say the $i'$th row $R = \br{((a,b),y) \in D \mid a=i}$.
We call a row $R$ of $D$ $r$-heavy if $|R| \geq |D|/r$, namely -- $R$ contains at least $|D|/r$ elements from $D$. Analogously, we define a column to be $r$-heavy.

Assume first that our $D$ contains a $\nu k$-heavy row $R$. We choose $R$
and use it as in the $1$-dimensional case. As explained above we can
find in $R$ a set of disjoint sub-signals $\B$ containing $\gamma
k$ blocks and for which the first item above holds for $\OPT_k(R)$ and in particular for $\OPT_k(D)$, i.e., $\sum_{B \in \B} \OPT_1(B) \leq \OPT_k(R) \leq \OPT_k(D)$. Further, using the above
guarantees for the $1$-dim case, $\abs{\cup_{B\in \B}B} \geq \frac{\gamma-2}{\gamma} \cdot |R| \geq \frac{\gamma-2}{\gamma} \cdot
\frac{|D|}{\nu k}$. This proves 2nd item for this case
(with $\gamma \geq 3$).

Otherwise, let $e_i = |R_i|$ where $R_i = \br{((a,b),y) \in D \mid a=i}$. By our assumption $e_i \leq |D|/\nu k $ for every $i \in [n]$. Our
algorithm is essentially identical to the $1$-dimensional case, on
the $1$-dim array $L=(e_1, \ldots , e_{n_1})$ where we weight the $i$th element by $e_i$. Namely, we find a partition of $L$ into $\psi$ contiguous
subintervals $\mathcal{E} = \br{E_1,\cdots,E_\psi}$ such that the corresponding sub-signals $D_j = \br{((a,b),y) \in D \mid a \in E_j}$ of $D$ are as equal as possible.
% $\mathcal{I} = \{I_1, \ldots, I_\psi\}$ %$\mathcal{I} = I_i = [a_i,b_i]$ where $a_1 = e_1, ~ b_t = e_{n_1}$ and
% if $b_i = e_\ell$ then $a_i = e_{\ell +1}$ and
% such that the sum of the entries in the intervals are as equal as possible. 
By our assumption this could be done so that for any $j \in \psi$, the number of elements in $D_j$ is between $\frac{|D|}{\nu k}$ and
at most $\frac{2|D|}{ \nu k}$. This is since adding a
new `point' from the list to an existing interval may increase the sum
by at most $|D|/\nu k$.
this implies that $\nu k/2 \leq \psi \leq \nu k$. 

Next we perform the above algorithm again on $D_1,\cdots,D_{\psi}$ with the intention to ``vertically partition'' each such $D_j$, i.e., split each $D_j$ into sets according to the value $b$ for every $((a,b),y) \in D_j$ (rather than considering the value $a$ above).
Let $r=2\nu^2 k^2$, we continue with the following case analysis: (i) at least a $\frac{1}{2}$-fraction of $\br{D_1,\cdots,D_\psi}$ contain no a $r$-heavy column, and (ii) at most a $\psi/2$ of $\br{D_1,\cdots,D_\psi}$ contain no a $r$-heavy column.
%Indeed setting a parameter $r=2\nu^2 k^2$, this can be done for each submatrix $A_i$ that contains no $r=2 \nu^2 k^2$-heavycolumn.

\textbf{Case (i): }At least a $\frac{1}{2}$-fraction of $\br{D_1,\cdots,D_\psi}$ contain no a $r$-heavy column. Then we
partition each set $D_j$ with no $r$-heavy column into
$\psi_i$ sub-signals $\br{D_j^{(1)},\cdots,D_j^{(\psi_i)}}$ of nearly equal number of points, where the partition is applied onto the values $b$ of every $((a,b),y) \in D_j$. By a reasoning similar to that above, each $r/2 \leq \psi_i \leq r$, and each such block $B \in \br{D_j^{(1)},\cdots,D_j^{(\psi_i)}}$ contains $|D_i|/r \leq |B|
\leq 2|D_i|/r$. Using the bounds on $|D_i|$ and $r$ we get $\frac{|D|}{2\nu^3  k^3} \leq|B| \leq
\frac{2|D|}{\nu^3 k^3}$. In particular we conclude that the total
number of such sub-signals is at most $t'$, and  $ t' \leq 2\nu^3
k^3$. On the  other hand, $t' \geq  \frac{\psi}{2}
\frac{|D|}{\nu k} \cdot \frac{\nu^3 k^3}{2|D|} \geq  \nu^3 k^3 /8$.
Let $\B_2$ be the collection of these sub-signals.

%For each signal $(B,\ell)$ where $B \in \mathcal{B}_2$ we compute $\OPT_1(B,\ell)$ and 
We choose for our output collection, $\B$, the set of $t'-z$ sub-signals $B \in \B_2$ with the smallest $\OPT_1(B)$, for $z = 2k(r + \psi) \leq 6k^3 \nu^2 $.

We note that $\B$ contains $t$ sub-signals, where $t' - z \leq t \leq t'$.  
Further, by the lower bound on $t$ it follows that $\abs{\cup_{B\in \B}B} \geq (t' -z) \cdot  \frac{|D|}{2\nu^3  k^3} \geq (\nu^3 k^3/8 - 6k^3 \nu^2  ) \frac{|D|}{2\nu^3  k^3} \geq \frac{|D|}{8} (1- 48/\nu) $.  

This verifies the 2nd item for this case.
Finally, we note that any row of $D$ is shared by at most $r$ sub-signals of $\B$, and each column of $D$ is shared by at most $\psi$ sub-signals of $\B$. Hence, any $k$-segmentation function may intersect at most $z = 2k(r +\psi)$ sub-signals from $\B$. Therefore, for any $k$-segmentation $s$ there are at least $t'-z$ sub-signals in $\B_2$ which are not-intersected by $s$. By our definition of $\B$ to be the set of $t'-z$ sub-signals in $\B_2$ with the smallest loss, we obtain that the loss $\cost(D,s)$ is at least $\sum_{B \in \B}\cost(B)$ proving the 1st item of the lemma for this case.

\textbf{Case (ii): }At most $\psi/2$ of $\br{D_1,\cdots,D_\psi}$ contain no a $r$-heavy column, namely -- at least $\psi/2$ of the $D_i$ have a $r$-heavy column. In this case we take the heavy column from each $D_i$ as its own sub-signal. We get a collection $\B_1$ of $\psi_1 \geq \psi/2$ blocks.
We now return as output the set $\B$ of the $\psi_1-2k$ sub-signals $B \in \B_1$ with the smallest $\OPT_1(B)$. We note that number of blocks we output in this case is at most $\psi \leq \nu k$.

Note also that $\B$ contains at least $\psi/2$ blocks, it
follows that $|\cup_{B\in \B}B|) \geq \frac{\psi}{2} \cdot
\frac{|D|}{2\nu^2 k^2} \geq \frac{|D|}{8 \nu k}$ which proves the 2nd
item in the lemma.

Finally, note that any $k$-segmentation $s$ can intersect at most $2k$
intervals from $\B_1$ (similarly to the $1$-dim case). Hence, there are at least $|\B_1|-2k = \psi_1-2k$ sub-signals in $\B_1$ that are not-intersected by $s$, which implies that its loss $\cost(D,s)$ is at least the sum $\sum_{B \in \B} \OPT_1(B)$, which proves the 1st item of this lemma.

{\bf Remark:  } we did not optimise the parameter. A slightly better
partition can be obtains (less blocks), but this is good enough for our purposes.

\textbf{Computational time: }Note that the elements of the input $D$ can be sorted in lexicographic order in $O(nm) = O(N)$ time since the coordinates $a$ and $b$ for every $((a,b),y) \in D$ are bounded integers.
Then, a linear-time preprocessing can be applied to the input $D$ to store some statistics, e.g., the number of elements in each non-empty row or column, and the index of the next non-empty row or column for every element in $D$. Afterwards, the above greedy partition also takes $O(nm) = O(N)$ time.
\end{proof}
\fi

We now restate and prove Lemma~\ref{lem:ab} from Section~\ref{sec:bicretiria}.
\begin{lemma} [Lemma~\ref{lem:ab}] \label{lem:ab_appendix}
Let $D = \br{(x_1,y_1),\cdots,(x_N,y_N)}$ be an $n\times m$ signal and $k \geq 1$ be an integer. Then, in $O(kN)$ time we can compute an $(\alpha,\beta)_k$-approximation for $D$, where $\alpha \in O(k\log N)$ and $\beta \in  O(k^{O(1)} \log^2{N})$.
\end{lemma}
The proof of the above claim is a constructive proof. A suggested implementation is provided in Algorithm~\ref{Alg:bicriteria}.
\ifproofs
\begin{proof}
The top level idea of the algorithm is as follows. We suggest an iterative algorithm. We start the first iteration with $D_1 = D$ and, using Lemma
\ref{lem:partition}, find a collection of disjoint sub-signals
$\B_1 = \{ B_1, \cdots, B_t\}$ in $D_1$ (which will not necessarily cover
the entire signal $D_1$), such that: (i) the sum of their $1$-segmentation loss satisfies,
$\sum_{i=1}^t \OPT_1(B_i) \leq \OPT_k(D_1) = \OPT_k(D)$, and (ii) $\cup_{B \in \B_1}B$ has size $\abs{\cup_{B \in \B_1}B} \geq
|D_1|/c$ for some $c$ (e.g., $c \in O(k)$ in the lemma). Let $\B_1$ be such a collection. We then `delete' from $D_1$ the elements of $\cup_{B \in \B_1}B$, and set $D_2 = D_1 \setminus \cup_{B \in \B_1}B$.

In the $i$th iteration, we repeat the same process with respect to the current set $D_i$. Namely, we find a collection $\B_i$ of at most $t$ disjoint sub-signals in $D$, for which: (i) $\sum_{B \in \B_i} \OPT_1(B) \leq \OPT_k(D_i) \leq \OPT_k(D)$, and (ii)  $\cup_{B \in \B_i} B$ has size $ \abs{\cup_{B \in \B_i} B} \geq |D_{i}|/c$, i.e., those blocks cover at least a constant fraction of $D_i$. 

Repeating these iterative procedure for at most $\psi = O(c\log (nm))$ times, we end up covering all entries of $D$ with sub-signals where the overall loss of the sub-signals in each iteration is at most $\OPT_k(D)$.
This defines a collection of at most $t \psi$ sub-signals $\B$ that cover the entire original set $D$. Hence, the total overall loss over those sub-signals is $\sum_{B \in \B'}OPT_1(B) \leq \psi \OPT_k(D)$.
By defining the output function $s$ such that for every $B \in \B$ and $b \in B$, $s$ assigns to $b$ the mean value of $B$, i.e., $s(b) = 1/\abs{B} \sum_{((i,j),y) \in B}y$, we obtain that $\cost(B,s) = \OPT_1(B)$ for every $B \in \B$, and that the $\cost(D,s) \leq \sum_{B \in \B}\cost(B,s) \leq \psi \OPT_k(D)$.

While for every $B \in \B$ $s$ assigns the same value for every element $b \in B$, there is not necessarily a partition of $n\times m$ into $|\B| \in O(t \psi)$ distinct axis-parallel blocks that correspond to the sub-signals of $\B$. Therefore, $s$ is not necessarily a $|\B|$-segmentation function.
However, looking at all possible intersections of the sub-signals in $\B$, it is easy to realize that the $t \psi$ sub-signals in $\B$ define a partition of $D$ into at most $O(t^2\psi^2)$ sub-signals $\B'$ that indeed correspond to a distinct partition of $n\times m$ into $|\B'|$ distinct axis-parallel blocks.
Hence, $s$ is guaranteed to be a $|\B'|$-segmentation function.

The parameters that are guaranteed by the lemma are $c \in O(k)$, $t \in O(k^3)$.  This implies that $\beta = |\B'| \in O(t^2 \psi^2) = O(k^8 \log^2 nm )$ and $\alpha = \psi \in O(k\log nm)$. 

\textbf{Computational time: }By Lemma~\ref{lem:partition}, each of the $\psi$ iterations above takes time linear in the input size. The input size in the $i$'th iteration is $O(N((k-1)/k)^i)$ since at each iteration we remove at least a $1/k$ fraction of the input. Hence, the total running time is the sum of the geometric series $N \cdot \sum_{i \in \psi}((k-1)/k)^i \in O(kN)$.
\end{proof}
\fi

\section{Balanced Partition} \label{sec:balancedPartition_appendix}
In this section we give our full proof for Lemma~\ref{PTpartition2D}.
We first prove the following lemma regarding the output of Algorithm~\ref{PTpartitionOneDim}.
\begin{lemma} \label{lem:colPartition}
Let $D$ be an $n\times m$ sub-signal, and $\sigma>0$ be a parameter. Let $\B = \br{B_1,\cdots,B_{|\B|}}$ be the output of a call to $\partitionAlgOneD(D,\sigma)$, where the sub-signals in $\B$ are numbered according to the order in which each of them was added to $\B$; see Algorithm~\ref{PTpartitionOneDim}. Then the following properties hold:
\begin{enumerate}
    \item $\B$ is a partition of $D$.
    \item $\OPT_1(B) \leq \sigma$ for every sub-signal $B \in \B$.% either $\left|\br{x \mid (x,y) \in B}\right| = 1$, or .
    \item If $|\B| >8k$ then for any $k$-segmentation $s$ that does not horizontally intersect $D$ we have that $\cost(D,s) \geq \left(\frac{|\B|}{4} - 2k\right) \sigma$.
    % \item There are at least $|\B|/4$ different consecutive pairs $B_i, B_{i+1} \in \B$ which satisfy $\OPT_1(B_i \cup B_{i+1}) > \sigma$.
    \item $\B$ can be computed in $O(|D|)$ time.
\end{enumerate}
\end{lemma}
\ifproofs
\begin{proof}
We consider the variables defined in Algorithm~\ref{PTpartitionOneDim}. 

\textbf{Proof of (i): }By construction it immediately follows that $\B$ is a partition of $D$. 

\textbf{Proof of (ii): }Consider a sub-signal $B \in \B$. We prove (ii) for each of the following cases: 
% Case (a) $B$ was either added to $\B$ at Line~\ref{line:addToOutput1D_1}, 
\textbf{Case (a):} $B$ was added to $\B$ at Line~\ref{line:addToOutput1D_3}, 
% Case (c) $B$ was either added to $\B$ at Line~\ref{line:addToOutput1D_4}, 
and \textbf{Case (b): }$B$ was either added to $\B$ at Line~\ref{line:addToOutput1D_2}

% \textbf{Case (a): }In this case, by the condition at Line~\ref{line:condSizeIs1}, $B$ must satisfy that $\left|\br{x \mid (x,y) \in B}\right| = 1$.

\textbf{Case (a): }In this case, by the condition at Line~\ref{line:condOptIsBig}, $B$ must satisfy that $\OPT_1(B) \leq \sigma$.

\textbf{Case (b): }In this case, $B$ was returned via a recursive call. Hence, this case holds trivially by Case (a) above.

Therefore, (ii) above holds by combining Cases (a)--(b).

\textbf{Proof of (iii): }Let $t = |\B|$ and assume for simplicity that $t$ is an even number. Recall that the index of each sub-signal in $\B$ indicates its order of insertion to $\B$, i.e., $B_1$ is the first sub-signal that was inserted to $\B$ and $B_t$ was the last such sub-signal to be inserted to $\B$.

Observe that each recursive call $\B' := \partitionAlgOneD(B^T, \sigma)$ at Line~\ref{line:1dRecursive} returns at least $|\B'| \geq 2$ sub-signals. This is because the recursive call happens only when $\OPT_1(B^T) = \OPT_1(B) > \sigma$, which can only happen if $\left|\br{(i,j) \mid \left((i,j),y\right) \in B}\right| > 1$, i.e., $B^T$ exceeds the maximum tolerance, and can indeed be partitioned into sub-signals.
Hence, there are at least $t/4$ distinct pairs of consecutive sub-signals $B_i$ and $B_{i+1}$ that were both either computed via the recursive call or both were not computed via the recursive call. We now show that each such pair satisfies $\OPT_1(B_i \cup B_{i+1}) > \sigma$.

Consider a pair of consecutive sub-signals $B_i$ and $B_{i+1}$ that were both not computed via the recursive call at Line~\ref{line:1dRecursive}. Let $B' \subseteq B_{i+1}$ contain the elements $((i,j),y) \in B_{i+1}$ with the smallest value of $i$ over all elements of $B_{i+1}$. By the greedy partition loop at Line~\ref{line:condOptIsBig} we obtain that $\OPT_1(B_i \cup B') > \sigma$. 
We now have that
\[
\OPT_1(B_i \cup B_{i+1}) \geq \OPT_1(B_i \cup B') > \sigma,
\]
where the first inequality is by Claim~\ref{claim:var}.

Consider a pair of consecutive sub-signals $B_i$ and $B_{i+1}$ that were both computed via the recursive call at Line~\ref{line:1dRecursive}. Then, similarly to the previous argument, we obtain that $\OPT_1(B_i \cup B_{i+1}) > \sigma$.

Now, let $s$ be a $k$-segmentation that does not horizontally intersect $\B$, i.e., it does not horizontally intersect any $B \in \B$. By the definition of $s$, there might be at most $2k$ sub-signals in $\B$ which are vertically intersected by $s$. Hence, among the $t/4$ distinct consecutive pairs of sub-signals discussed above there are at least $t/4-2k$ such pairs that are not intersected by $s$.

Since $\B$ is a partition of $D$, we have that $\cost(D,s)$ is at least the sum of $\OPT_1(B_i \cup B_{i+1}) \geq \sigma$, over the above $t/4-2k$ non-intersected pairs of sub-signals. Hence,
\[
\cost(D,s) \geq \left(\frac{t}{4}-2k\right) \sigma = \left(\frac{|\B|}{4}-2k\right) \sigma.
\]

\textbf{Proof of (iv): }The greedy Algorithm~\ref{PTpartitionOneDim} can be implemented so that it computes only $O(|D|)$ operations. The most costly operation is the computation of $\OPT_1(B)$ for some sub-signal $B$. We now argue that this can be computed in $O(1)$ time.
Let $B$ be a sub-signal and let $\mu_B = 1/|B| \sum_{(x,y) \in B} y$ be its mean value. Observe that
\begin{equation} \label{eq:openSquared}
\begin{split}
\OPT_1(B) & 
= \sum_{(x,y) \in B}(y-\mu_B)^2\\=
& \sum_{(x,y) \in B} y^2 + |B| \cdot \mu_B -2 \mu_B \sum_{(x,y) \in B} y.
\end{split}
\end{equation}
By precomputing and storing some statistics at each of the signal's elements, then the three terms on the right hand side of~\eqref{eq:openSquared} can all be evaluated in $O(1)$ time for any sub-signal $B$.
Hence, the total running time of Algorithm~\ref{PTpartitionOneDim} is linear in the input size.

\end{proof}
\fi

We now restate and prove Lemma~\ref{PTpartition2D} from Section~\ref{sec:balancedPartition}.
\begin{lemma} \label{PTpartition2D_appendix}
Let $D$ be an $n\times m$ signal, $k\geq 1$ be an integer, $\varepsilon \in (0,1/4)$ be an error parameter, and
$s:[n]\times[m] \to \REAL$ be an $(\alpha,\beta)_k$-approximation of
$D$. Define $\sigma := \frac{\cost(D,s)}{\alpha}$ and $\gamma :=
\frac{\varepsilon^2}{\beta k}$. Then algorithm
$\partitionAlg(D,\gamma,\sigma)$ outputs a partition $\B$ of $D$ 
that is an $\left(O\left(\frac{\alpha}{\gamma^2}\right), \gamma^2\sigma, O\left(\frac{k\alpha}{\gamma}\right)\right)_k$-balanced partition in $O(|D|)$ time.
\end{lemma}
\ifproofs
\begin{proof}
To prove that $\B$ is a $\left(O\left(\frac{\alpha}{\gamma^2}\right), \gamma^2\sigma, O\left(\frac{k\alpha}{\gamma}\right)\right)_k$-balanced partition as in Definition~\ref{def:balancedPartition}, we need to prove the following properties:
\begin{enumerate}
    \item $\OPT_1\left(B\right) \leq \gamma^2\sigma$ for every $B \in \B$.% either  or $\left| \br{x \mid (x,y) \in B} \right| = 1$.
    \item $\B$ is a partition of $D$ whose size is $|\B| \in O\left(\frac{\alpha}{\gamma^2}\right)$.
    \item For every $k$-segmentation $\hat{s}$ there are  $O\left(\frac{k\alpha}{\gamma}\right)$ sub-signals $B \in \B$ for which $\hat{s}$ assigns at least $2$ distinct values, i.e., $\left|\br{\hat{s}(x) \mid (x,y) \in B}\right| \geq 2$.
\end{enumerate}

\textbf{Proof of (i): }Observe that the output set $\B$ contains the the union of multiple output sets $\B' := \partitionAlgOneD( \cdot , \gamma^2\sigma)$ computed via calls to Algorithm~\ref{PTpartitionOneDim}. 
By Property (ii) of Lemma~\ref{lem:colPartition}, every sub-signal $B \in \B'$ in such output set $\B'$ satisfies that $\OPT_1\left(B\right) \leq \gamma^2\sigma$. Hence, Property (i) of Lemma~\ref{PTpartition2D} immediately holds.

\textbf{Proof of (ii): }By the greedy construction it holds that $\B$ is a partition of $D$. 
We now prove that the number of times that Line~\ref{line:addToOutput} was executed is $t \in O(\alpha/\gamma)$, i.e., the number of times we append a set of signals $last\B'$ to $\B$ is at most $O(\alpha/\gamma)$. 
Let $\B_1 ,\cdots, \B_t$ denote the set of sub-signals $last\B'$ in each of the $t$ executions of Line~\ref{line:addToOutput}, i.e., $\B_1 := last\B'$ at the first time Line~\ref{line:addToOutput} was executed. Each such set is called a \emph{horizontal set}.

Recall that $s \in \S_{\beta k}$ is a $\beta k$-segmentation. 
First of all, by the definition of a $\beta k$-segmentation function there are at most $2\beta k$ sets among the horizontal sets $\B_1,\cdots, \B_t$ which can be horizontally intersected by $s$.

Consider two consecutive horizontal sets $\B_i, \B_{i+1}$ that are not horizontally intersected by $s$, and let $H = \bigcup_{B \in B_{i}\cup B_{i+1}}B$ be the union of all the sub-signals in $\B_i \cup \B_{i+1}$.
We now argue that the loss $\cost(H,s)$ is at least $O(\gamma \sigma)$. Since $\B_i$ and $\B_{i+1}$ are two different horizontal sets, by the greedy construction we know that their union $H$ could have been partitioned via a call to $\E := \partitionAlgOneD(H , \gamma^2\sigma)$ into a set $\E = \br{E_1,\cdots,E_{|\E|}}$ of at least $|\E| \geq \frac{1}{\gamma}$ blocks. 
By substituting $D = H, k=\beta k, \B = \E$ and $\sigma = \gamma^2 \sigma$ in Property (iii) of Lemma~\ref{lem:colPartition}, for a $\beta k$-segmentation $s$, we have that
\[
\begin{split}
\cost(H,s) & \geq \left(\frac{|\E|}{4} - 2\beta k\right) \gamma^2 \sigma
\geq \left(\frac{1}{4\gamma} -2\beta k\right) \gamma^2\sigma\\ 
& \geq \left(\frac{1}{4\gamma} -\frac{\beta k}{9 \varepsilon^2}\right) \gamma^2\sigma
= \left(\frac{1}{4\gamma} -\frac{1}{9\gamma}\right) \gamma^2\sigma \\
& \geq \gamma \sigma/2,
\end{split}
\]
where the second derivation holds for $\varepsilon \in (0,1/3)$, and the third derivation is by the definition of $\gamma$.

Assume by contradiction that there are more than $\frac{2\alpha}{\gamma}$ such pairs of consecutive horizontal sets $\B_i, \B_{i+1}$, which are not horizontally intersected by $s$. The loss of those slices to $s$ would be bigger than $\frac{2\alpha}{\gamma} \cdot \frac{\gamma\sigma}{2} = \alpha \sigma = \cost(D,s)$, which is a contradiction. Therefore, the number of pairs of consecutive horizontal sets, which are not horizontally intersected by $s$, cannot exceed $O\left(\frac{\alpha}{\gamma}\right)$. Observe that the total number of horizontal sets  that can be intersected by $s$ is at most $2\beta k$.
Hence, the total number of horizontal sets is at most
\begin{equation}\label{eq:boundm}
m \in O\left(\frac{\alpha}{\gamma} + 2\beta k\right) \in O\left(\frac{\alpha}{\gamma}\right).
\end{equation}

We now prove that the number of output cells is at most $|\B| \in O\left(\frac{\alpha}{\gamma^2}\right)$ in two steps. In step (i) we consider the horizontal sets that contain more than one row of $D$ and show that they contain a total of $O\left(\frac{\alpha}{\gamma^2}\right)$ sub-signals. In step (ii) we consider the horizontal sets that contain exactly one row of $D$ and prove that they also contain a total of $O\left(\frac{\alpha}{\gamma^2}\right)$ sub-signals.

\textbf{Step (i): }By~\eqref{eq:boundm}, the total number of horizontal sets is at most $m \in O(\frac{\alpha}{\gamma})$. Therefore, the total number of horizontal slices that contain more than one row of $A$ is also at most $O\left(\frac{\alpha}{\gamma}\right)$.
By the construction in Algorithm~\ref{PTpartitionAllDims}, each such horizontal set $\B_i$ with more than $1$ row of $A$ is partitioned into at most $\frac{1}{\gamma}$ sub-signals. Hence, the total number of blocks in horizontal sets than contain more than one row of $D$ is at most $O\left(\frac{\alpha}{\gamma} \cdot \frac{2}{\gamma}\right) = O\left(\frac{\alpha}{\gamma^2}\right)$.

\textbf{Step (ii): }Consider all the horizontal sets $\B_i$ which contain one row of $D$, and which have been partitioned into $|\B_i| \leq 2\beta k \leq 1/\gamma$ blocks. The total number of blocks in such horizontal slices is thus bounded by the maximum number of horizontal slices $m \in O(\alpha/\gamma)$ times $1/\gamma$ for a total of at most $O(\alpha / \gamma^2)$ blocks.

For the rest of this step, we assume that all horizontal sets $\B_i$ have $|\B_i| \geq 2\beta k$. Let $G \subseteq [t]$ contain the indices of the horizontal sets which contain exactly one row of $D$, and let $i\in G$. Observe that $\B_i$ was computed, at some point, via a call $\B_i := \partitionAlgOneD(\cup_{B \in \B_i}B, \gamma^2\sigma)$. Also, since the points $\cup_{B \in \B_i}B$ of $\B_i$ all have the same row index,  observe that $\B_i$ cannot be horizontally intersected by $s$. Therefore, by substituting $D = \cup_{B \in \B_i}B, k=\beta k$ and $\sigma = \gamma^2\sigma$ in Property (iii) of Lemma~\ref{lem:colPartition} we obtain that
\begin{equation} \label{eq:boundBj}
\cost(\B_i,s) \geq \left(\frac{|\B_i|}{4}-2\beta k\right)\cdot \gamma^2\sigma.
\end{equation}

Furthermore, we have that
\begin{equation} \label{eq:boundsummj}
\alpha \cdot \sigma = \cost(D,s) \geq \sum_{i \in G} \cost(\B_i,s) \geq \sum_{i \in G}\left(\frac{|\B_i|}{4}-2\beta k\right) \cdot \gamma^2\sigma,
\end{equation}
where the first derivation is by the definition of $\sigma$, the second derivation holds since $\br{(x,y) \in B \mid B \in \B_i, i\in G} \subseteq D$, and the third derivation is by~\eqref{eq:boundBj}.
Rearranging terms in~\eqref{eq:boundsummj} concludes Step (ii) as
\[
\begin{split}
\sum_{i \in G} |\B_i|& 
\leq \frac{4\alpha}{\gamma^2} + 8\sum_{i \in G}\beta k 
\leq \frac{4\alpha}{\gamma^2} + 8t \beta k \\
& \leq \frac{4\alpha}{\gamma^2} + \frac{8\alpha \beta k}{\gamma} 
\in O\left(\frac{\alpha}{\gamma^2}\right).
\end{split}
\]
Therefore, the total number of blocks in horizontal sets that contain exactly one row of $D$ is at most $O\left(\frac{\alpha}{\gamma^2}\right)$.

\textbf{Proof of (iii): }
By the properties above we have that: (i) there are at most $O(\alpha/\gamma)$ horizontal sets, and (ii) each horizontal set $\B_i$ either contains at most $O(1/\gamma)$ sub-signals, or all the points $((i,j),y) \in B$ of all the blocks $B \in \B_i$ have the same row index $i$. 
Let $\hat{s}$ be a $k$-segmentation. $\hat{s}$ can horizontally intersect all the (at most) $1/\gamma$ sub-signals of at most $k$ horizontal sets, and can vertically intersect at most $1$ block from each of the $O(\alpha/\gamma)$ horizontal sets. Hence, the total number of intersected sub-signals is $O(k\alpha/\gamma)$.

\textbf{Computational time: }We now prove that $\B$ can be computed in $O(|D|)$ time. The computational time of Algorithm~\ref{PTpartitionAllDims} is dominated by the computational time of Line~\ref{line:partitionSlice} where we partition a slice $S$. Using Algorithm~\ref{PTpartitionOneDim} we can partition each such slice $S$ in linear $O(|S|)$ time; see Lemma~\ref{lem:colPartition}. 
Therefore, the naive implementation, i.e. by calling Algorithm~\ref{PTpartitionOneDim} for every slice $S$, will result in $O(|D|^2)$ overall time, since many rows of $D$ participate many times in such a call to Algorithm~\ref{PTpartitionOneDim}.

However, we can implement Line~\ref{line:partitionSlice} in $O(m)$ time, rather than $O(|S|)$ time, by preprocessing the input signal $D$, in linear time $O(|D|)$, and storing some statistics for every element $((i,j),y) \in D$. For example, one can store the sum of values and squared values over all elements $((i',j'),y')$ where $i'< i$ or $j'<j$. Using those values we can compute $\OPT(B)$ in $O(1)$ time for every sub-signal $B$ of $D$. Now, using such statistics (and possibly more statistics), Line~\ref{line:partitionSlice} can be implemented in $O(m)$ time via a greedy algorithm that iterates over the points of the last row $R = \br{((i,j),y) \in D \mid i = r_{end}}$ added to $S$ (i.e. with no need to iterate over other elements of $S$). We leave the small details to the reader.

\end{proof}
\fi

\section{Coreset Construction} \label{sec:coresetConstruction_appendix}
In this section, we provide the proof of correctness for our main coreset construction algorithm presented in Algorithm~\ref{Coreset}; see Theorem~\ref{theorem:coreset_appendix}. 
Furthermore, we provide an algorithm than gets as input a $k$-segmentation $s$, as well as a $(k,\varepsilon)$-coreset for some input dataset $D$, which was computed using Algorithm~\ref{Coreset}. The algorithm returns a $(1+\varepsilon)$-approximation to the loss $\cost(D,s)$, in $O(k|C|)$ time; see Algorithm~\ref{CoresetQuery} and full details in Lemma~\ref{lemma:coresetQuery}.

In what follows, for an $n\times m$ sub-signal $B$ and a weight function $u:B\to[0,\infty)$, we abuse notation and denote $u((a,b))$ by simply $u(a,b)$ for $(a,b) \in B$.

\textbf{Some intuition behind Algorithm~\ref{CoresetQuery}. }Given a $(k,\varepsilon)$-coreset $(C,u)$ for an input dataset $D =\br{(x_1,y_1),\cdots,(x_{N},y_{N})}$, and a $k$-segmentation $s$, the algorithm outputs a $(1+\varepsilon)$-approximation to $\cost(D,s)$ in time that depends only on $k$ and $|C|$.

During the computation of $(C,u)$ in Algorithm~\ref{Coreset}, a partition $\B$ of $D$ was computed. Then, for every set $B$ in the partition $\B$, a representative pair $(C_B,u_B)$ for $B$ was computed and added to $C$. 

To approximate the loss $\cost(D,s)$, we will approximate individually $\cost(B,s)$ for every $B \in \B$, and return the sum of those losses.
Therefore, we now consider a single set $B \in \B$, and consider the following two cases. 

\textbf{Case (i) : }$s$ assigns the same value for all the elements of $B$. Then, by construction, it is guaranteed that $\cost(B,s) = \sum_{(x,y) \in C_B} u_B(x,y)(s(y)-y)^2$. Therefore, in this case, $\cost(B,s)$ will be accurately estimated using $(C_B,u_B)$.

\textbf{Case (ii) : }$s$ assigns more than one unique value to the elements of $B$. In this case, if we ignore the computational time, we would ideally want to compute a ``smoothed version'' $(S,w)$ of $(C_B,u_B)$, as shown in Fig.~\ref{fig:blockCoreset} (see~\eqref{eq:propS}-~\eqref{eq:propSSum} below for formal details). Then, we would return the loss $\sum_{(x,y) \in S} w(x,y)(s(y)-y)^2$. However, computing $(S,w)$ is not necessary, since there are many subsets of $B$ in which all the elements $x \in B$ have simultaneously the same label in $S$ and are assigned the same value by $s$. Combining this with the fact that those subsets are of rectangular (simple) shape, we obtain that the loss over those subsets can be evaluated efficiently, as computed in Algorithm~\ref{CoresetQuery}.

\begin{algorithm}[h]
    \caption{\textsc{$\coresetQuery((C,u), s)$}; see Lemma~\ref{lemma:coresetQuery}}
    \label{CoresetQuery}
    \SetKwInOut{Input}{Input}
	\SetKwInOut{Output}{Output}
    \Input{A $(k,\varepsilon)$-coreset $(C,u)$ which was returned from a call to $\coresetAlg(D,k,\varepsilon/\Delta)$ in Algorithm~\ref{Coreset}, for some $n\times m$-signal $D$, $k\geq 1$, $\varepsilon \in (0,1)$ and a sufficiently large $\Delta \geq 1$.\\
    A $k$-segmentation (or $k$-tree) $s$.}
    \Output{A $(1+\varepsilon)$-approximation to the loss $\cost(D,s)$.}

    $loss := 0$
    
    \For{every $4$ consecutive elements $\hat{C} = \br{(a_i,b_i)}_{i=1}^4$ in $C$\label{line:forBlockCoreset}} 
    {
        Denote by $B$ the sub-signal that corresponds to $C'$. \label{line:correspondingB}\tcp{By construction in Algorithm~\ref{Coreset}, the coordinates $a$ of the $4$ elements $(a,b) \in \hat{C}$ are the corners of $B$.}
        
        $z := |\br{s(x) \mid (x,y) \in B}|$ \label{line:compNumIntersection}
        
        \If {$z = 1$ \tcp{i.e., $s$ does not intersect $B$}} 
        {
            $loss_{\hat{C}} := \sum_{(x,y) \in C_B} u_B(x,y)(s(x)-y)^2$. \label{line:compLossCase1} \tcp{note that $s(x_1) = s(x_2)$ for every $x_1,x_2 \in B$} \label{line:ifZEq1}
        } 
        \Else 
        { \tcp{In this case, $s$ intersects $B$}
        
            Denote by $S$ the partition that $s$ induces onto $[n]\times [m]$. \label{line:startSecondCase} \tcp{$S$ contains $|S| \leq k$ subsets of $[n]\times [m]$.}
            
            $i := 1$
            
            \For{every $S' \in S$ \label{line:forEverySPrimeInS}} {
                Denote by $\ell$ the label that $s$ assigns to the elements of $S'$ i.e., $s(x,y) = \ell$ for every $(x,y) \in S'$
                
                $z := \abs{B \cap S'}$ \tcp{The number of element in the intersection of the $S'$ and the subset of $[n]\times [m]$ that is represented by $C'$.}
                
                $loss_{\hat{C}} := 0$
                
                \While{$z \geq 1$}
                {
                    \If{$u(a_i,b_i) \leq z$}
                    {
                        $loss_{\hat{C}} := loss_{\hat{C}} + u(a_i,b_i) \cdot (\ell-b_i)^2$ \label{line:LossCAddIf}
                        
                        $u(a_i,b_i) := 0$
                        
                        $z := z - u(a_i,b_i)$
                        
                        $i := i + 1$ \label{line:WhileZGeq1IncrementI}
                    } \Else {
                        $loss_{\hat{C}} := loss_{\hat{C}} + z \cdot (\ell-b_i)^2$ \label{line:LossCAddElse}
                        
                        $u(a_i,b_i) := u(a_i,b_i)- z$
                        
                        $z := 0$ \label{line:WhileZGeq1ZeroZ}
                    }
                }
                $loss := loss + loss_{\hat{C}}$. \label{line:addToLoss}
            }
        }
    }

    \Return $loss$
\end{algorithm}

\begin{lemma} \label{lemma:coresetQuery}
Let $D =\br{(x_1,y_1),\cdots,(x_{N},y_{N})}$ be an $n\times m$ signal i.e., $N := nm$. Let $k\geq 1$ be an integer, $\varepsilon \in (0,1/4)$ be an error parameter, and $(C,u)$ be the output of a call to $\coresetAlg(D,k,\varepsilon)$ (see Algorithm~\ref{Coreset}). Let $s$ be a $k$-segmentation (in particular, a $k$-tree). Finally, let $loss$ be the output of a call to $\coresetQuery((C,u),s)$; see Algorithm~\ref{CoresetQuery}. Then there is a sufficiently large constant $\Delta\geq 1$ such that
\[
\abs{\cost(D,s) - loss} \leq \Delta \varepsilon \cdot \cost(D,s).
\]
Moreover, $loss$ can be computed in $O(k|C|)$ time.
\end{lemma}
\begin{proof}
We consider the variables defined in Algorithm~\ref{CoresetQuery}.

First, consider a subset $\hat{C}$ of $C$ from some iteration of the For loop at Line~\ref{line:forBlockCoreset} of Algorithm~\ref{CoresetQuery}, and let $B$ be the sub-signal that corresponds to $\hat{C}$, as in Line~\ref{line:correspondingB}.
We now prove that the loss $loss_{\hat{C}}$ computed in the same iteration of the For loop (i.e., at Lines~\ref{line:correspondingB}-~\ref{line:addToLoss}) satisfies the following claim.
\begin{claim} \label{claim:oneBlock}
Let $z = |\br{s(x) \mid (x,y) \in B}|$ be the number of distinct values $s$ assigns to the coordinates of $B$ (as computed in Line~\ref{line:compNumIntersection}). Then,
$loss_{\hat{C}}$ satisfies that
\[
\begin{cases}
\cost(B, s) = loss_{\hat{C}} & \text{if } $z = 1$\\
\left|\cost(B, s) - loss_{\hat{C}}\right| \leq \varepsilon \cdot \cost(B, s) +  O\left(\frac{\OPT_1(B)}{\varepsilon}\right) & \text{otherwise} 
\end{cases}
\]
\end{claim}
\begin{proof}
We prove Claim~\ref{claim:oneBlock} using the following case analysis: (i) $z=1$ and (ii) $z \geq 2$.

\textbf{Case (i): }$z=1$. 
We prove that $\cost(B,s) = loss_{\hat{C}}$.

Since the input coreset $(C,u)$ was computed using Algorithm~\ref{Coreset}, we know that the set $\hat{C}$ was computed at Line~\ref{line:compBlockCoreset} of Algorithm~\ref{Coreset}, along with a weight function $\hat{u}$. Hence, the pair $(\hat{C},\hat{u})$ satisfy, by construction, the following property:
\begin{equation} \label{eq:propHatC}
\sum_{(a,b) \in \hat{C}} \hat{u}((a,b))\cdot (b \mid b^2 \mid 1) = \sum_{(x,y) \in B}(y \mid y^2 \mid 1).
\end{equation}

Now, for any constant $\hat{s} \in \REAL$, we have that
\begin{equation} \label{eq:1meanCoreset}
\begin{split}
& \sum_{(a,b) \in \hat{C}} \hat{u}(a,b) (b-\hat{s})^2\\
& = \sum_{(a,b) \in \hat{C}} \hat{u}(a,b) \cdot b^2 + \hat{s}^2 \sum_{(a,b) \in \hat{C}} \hat{u}(a,b)
-2\hat{s}\sum_{(a,b) \in \hat{C}} \hat{u}(a,b) b\\
& = \sum_{(x,y) \in B} y^2 + \hat{s}^2 \sum_{(x,y) \in B} 1 -2\hat{s}\sum_{(x,y) \in B} y\\
& = \sum_{(x,y) \in B} (y-\hat{s})^2,
\end{split}
\end{equation}
where the second equality is by~\eqref{eq:propHatC}.

Since $z = |\br{s(x) \mid (x,y) \in B}| = 1$, there is a constant $\hat{s} \in \REAL$ such that 
\begin{equation} \label{eq:oneMeanCost}
\cost(B, s) = \sum_{(x,y) \in B} (y-\hat{s})^2. 
\end{equation}
Hence, we have that 
\[
\cost(B,s) = \sum_{(x,y) \in B} (y-\hat{s})^2 =\sum_{(a,b) \in \hat{C}} \hat{u}(a,b) (b-\hat{s})^2 = loss_{\hat{C}},
\]
where the first derivation is by~\eqref{eq:oneMeanCost}, the second derivation is by~\eqref{eq:1meanCoreset}, and the last derivation is by the definition of $loss_{\hat{C}}$ at Line~\ref{line:compLossCase1} of Algorithm~\ref{CoresetQuery}.

\textbf{Case (ii): }$z \geq 2$. 
We prove that
\[
\left|\cost(B, s) - loss_{\hat{C}}\right| \leq \varepsilon \cdot \cost(B, s) +  O\left(\frac{\OPT_1(B)}{\varepsilon}\right).
\]

We first observe that, by the triangle inequality, for any $a,b,c \in \REAL$ we have that
\begin{equation} \label{eq:epsTriangleIneqality}
\begin{split}
||a-c|^2 - |b-c|^2| 
& = ||a-c| - |b-c|| \cdot \left( |a-c| + |b-c| \right) \\
& \leq |a-b| \cdot (2|a-c| + |a-b|) \\
& = |a-b|^2 + 2|a-c|\cdot|a-b| \\
& = |a-b|^2 + 2\sqrt{\varepsilon}|a-c|\cdot \frac{|a-b|}{\sqrt{\varepsilon}} \\
& \leq |a-b|^2 + \varepsilon\cdot |a-c|^2 + \frac{|a-b|^2}{\varepsilon} \\
& = \varepsilon \cdot |a-c|^2 + \left(1 + \frac{1}{\varepsilon}\right) \cdot (a-b)^2, 
\end{split}
\end{equation}
where the second inequality holds since $2xy \leq x^2 + y^2$ for every $x,y \in \REAL$.

\textbf{Smoothed coreset. }
In Algorithm~\ref{Coreset} we computed some small compression $C_B$, along with a weights function $u_B$, for every subset $B$ in the partition of the input. The size $|C_B|$ of this compression is a small constant, independent of the (potentially large) size of $B$. The pair $(C_B,u_B)$ satisfy a set of properties, which we visually demonstrate via this ``smoothed coreset'' notion; see Fig.~\ref{fig:blockCoreset}. Informally, the ``smoothed version'' of $(C_B,u_B)$ is another pair $(C_B’,u_B’)$, such that $C_B’$ contains a duplication of the elements of $C_B$. The number of duplications of every element $c$ from $C_B$ is according to its weight $u_B(c)$.

We now formally define a ``smoothed version'' of a pair $(\hat{C}, \hat{u})$.
A pair $(S, w)$ is said to be a \emph{smoothed version} of the pair $(\hat{C}, \hat{u})$ if it satisfies the following properties:
(i) $(S,w)$ has the same sum of weights, sum of labels, and sum of squared labels as $(\hat{C},\hat{u})$,
(ii) The set of coordinates $\br{a | (a,b) \in S}$ in $S$ covers the entire set of coordinates $\br{x | (x,y) \in B}$ of the original set $B$, with possible duplicates, and
(iii) The sum of weights over all elements in $S$ with the same coordinate is $1$. Formally,
\begin{equation} \label{eq:propS}
\sum_{(a,b) \in S} w((a,b))\cdot (b \mid b^2 \mid 1) = \sum_{(a,b) \in \hat{C}}\hat{u}((a,b))(b \mid b^2 \mid 1),
\end{equation}
\begin{equation}\label{eq:propSNoHoles}
\br{x | (x,y) \in B} = \br{a | (a,b) \in S},
\end{equation}
and
\begin{equation}
\label{eq:propSSum}
\sum_{(a,b) \in S: a=x} w((a,b)) = 1 \text{ for every }(x,y) \in B.
\end{equation}

In what follows, for every pair $(S,w)$ which is a smoothed version of $(\hat{C},\hat{u})$, we prove the following two properties:
We now prove the following two properties: first, that
\begin{equation} \label{eq:smoothedCoresetProp}
\left|\cost(B, s) - \cost((S,w), s)\right| \leq \varepsilon \cdot \cost(B, s) +  O\left(\frac{\OPT_1(B)}{\varepsilon}\right),
\end{equation}
for every every pair $(S,w)$ which is a smoothed version of $(\hat{C},\hat{u})$. Second, we need to prove there is some pair $(\hat{S},\hat{w})$ which is a smoothed version of $(\hat{C},\hat{u})$ that satisfies
\begin{equation} \label{eq:lossSW}
loss_{\hat{C}} = \cost((\hat{S},\hat{w}), s),
\end{equation}
where $loss_{\hat{C}}$ is the loss computed at Lines~\ref{line:correspondingB}-~\ref{line:addToLoss}, using only the pair $(\hat{C},\hat{u})$ (i.e., at the current iteration of the outer-most For loop of Algorithm~\ref{CoresetQuery}), without actually computing $(\hat{S},\hat{w})$).
Case (ii) then immediately holds by combining~\eqref{eq:smoothedCoresetProp} and~\eqref{eq:lossSW} above.

\paragraph{A proof of~\eqref{eq:smoothedCoresetProp}.}
Let $(S,w)$ be a pair which is a smoothed version of $(\hat{C},\hat{u})$.
By definition of $(S,w)$, we have that $w$ sums to $1$ over all $(a,b) \in S$ with the same $a$, as in~\eqref{eq:propSSum}. Therefore, for every $(x,y)\in B$ we can rewrite the term $(y - s(x))^2$ as $\sum_{(a,b) \in S: a = x} w(a,b) (y - s(x))^2$. Now, define $y_B(x) = y$ for every $(x,y) \in B$. We therefore have that
\begin{equation} \label{eq:uSumsTo1}
\begin{split}
\cost(B,s) & = \sum_{(x,y) \in B} (y - s(x))^2\\ 
& = \sum_{(x,y) \in B} \left(\sum_{(a,b) \in S: a=x} w(a,b)\right) \cdot (y - s(x))^2\\ 
& = \sum_{(x,y) \in S} w(x,y) (y_B(x) - s(x))^2,
\end{split}
\end{equation}
where the last equality holds by \eqref{eq:propSNoHoles} and by simply combining the two sums.

We now have that
\begin{align}
& \left|\cost(B, s) - \cost((S, u),s)\right|\nonumber\\ 
& = \Bigg| \sum_{(x,y) \in S} w(x,y) (y_B(x) - s(x))^2 - \sum_{(x,y) \in S} w(x,y) (y-s(x))^2\Bigg| \label{eq:splitBlock2}\\
& = \left| \sum_{(x,y) \in S} w(x,y)\cdot \left((y_B(x) - s(x))^2 - (y - s(x))^2 \right)\right| \nonumber\\
& \leq \sum_{(x,y) \in S} w(x,y)\left| (y_B(x) - s(x))^2 - (y - s(x))^2\right| \label{eq:splitBlock3}\\
& \leq \sum_{(x,y) \in S} w(x,y)\Bigg( \varepsilon \cdot (y_B(x) - s(x))^2 + \left(1+\frac{1}{\varepsilon}\right)(y_B(x)  - y)^2\Bigg) \label{eq:splitBlock4}\\
& = \varepsilon \cdot \sum_{(x,y) \in S} u(x,y)\cdot (y_B(x) - s(x))^2 + \left(1+\frac{1}{\varepsilon}\right) \sum_{(x,y) \in S}  w(x,y)\cdot (y_B(x)  - y)^2 \nonumber\\
& = \varepsilon \cdot \cost(B,s) + \left(1+\frac{1}{\varepsilon}\right) \sum_{(x,y) \in S}  u(x,y)\cdot (y_B(x)  - y)^2, \label{eq:splitBlock6}
\end{align}
where~\eqref{eq:splitBlock2} is by combining~\eqref{eq:uSumsTo1} and the definition of $\cost$, \eqref{eq:splitBlock3} holds since the sum of absolute values is greater or equal than the absolute value of a sum, \eqref{eq:splitBlock4} holds by substituting in~\eqref{eq:epsTriangleIneqality} every term in the sum, and~\eqref{eq:splitBlock6} is by~\eqref{eq:uSumsTo1}.

 We now bound the rightmost term of~\eqref{eq:splitBlock6}. Let $\hat{s} \equiv 1/|B| \sum_{(x,y)\in B}y$ be a $1$-segmentation function that returns the label mean of $B$. We have that
\begin{align}
& \sum_{(x,y) \in S}  w(x,y)\cdot (y_B(x)  - y)^2\nonumber\\
& \leq 2 \cdot \sum_{(x,y) \in S} w(x,y)\cdot \left( (y_B(x) - \hat{s}(x))^2 + (y - \hat{s}(x))^2 \right) \label{eq:boundRightTerm1}\\
& = 2 \sum_{(x,y) \in S} w(x,y)\cdot (y_B(x) - \hat{s}(x))^2\nonumber\\ & \quad\quad+ 2 \sum_{(x,y) \in S} w(x,y)\cdot (y - \hat{s}(x))^2 \nonumber\\
& = 2 \cdot(\cost(B, \hat{s}) + \cost((S,u), \hat{s})) \label{eq:boundRightTerm152}\\
& = 2 \cdot(\cost(B, \hat{s}) + \cost(B, \hat{s})) \label{eq:boundRightTerm2}\\
& = 4 \cdot\cost(B, \hat{s}) \nonumber\\
& = 4 \cdot\OPT_1(B) \label{eq:boundRightTerm4}
% & \leq 4 \cdot\gamma^2\sigma, \label{eq:boundRightTerm5}
\end{align}
where~\eqref{eq:boundRightTerm1} is by the weak triangle inequality, \eqref{eq:boundRightTerm152} is by combining the definition of $\cost$ with~\eqref{eq:uSumsTo1}, \eqref{eq:boundRightTerm2} holds by Case (i) above, and \eqref{eq:boundRightTerm4} holds since the label means minimizes its sum of squared differences to the labels.

Equation~\eqref{eq:smoothedCoresetProp} now holds by combining~\eqref{eq:splitBlock6} and~\eqref{eq:boundRightTerm4}.

\paragraph{A proof of~\eqref{eq:lossSW}.}
To prove~\eqref{eq:lossSW}, in Fig.~\ref{fig:smoothed_coreset_proof} we construct a smoothed version $(\hat{S}, \hat{w})$ of $(\hat{C},\hat{u})$ which satisfies~\eqref{eq:lossSW}. 

Furthermore, by combining the construction of $(\hat{S}, \hat{w})$ with the computation of $loss_{\hat{C}}$ in Lines~\ref{line:startSecondCase}-\ref{line:addToLoss} of Algorithm~\ref{CoresetQuery}, we obtain, as desired, that $loss_{\hat{C}} = \cost((\hat{S},\hat{w}), s)$.

\begin{figure}
    \centering
    \includegraphics[width=\textwidth]{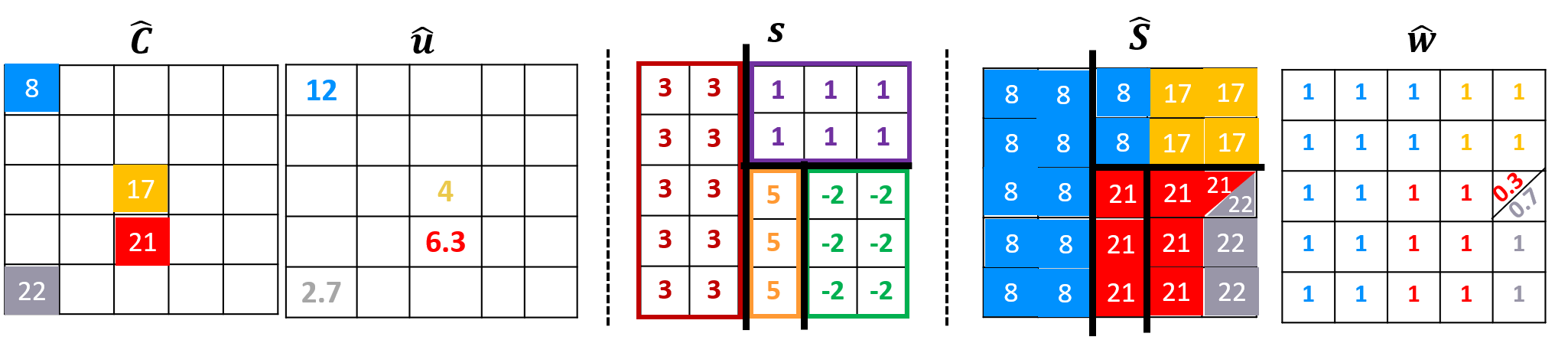}
    \caption{\textbf{(Left):} The pair $(\hat{C},\hat{u})$. \textbf{(Middle): }A $4$-segmentation $s$, which induces a partition of $[5]\times[5]$ into $4$ sets $\B = B_1 \cup B_2 \cup B_3 \cup B_4$ where $B_1 = \br{(1,1),(2,1),(3,1),(4,2),\cdots,}$, $B_2 = \br{(1,3),(2,3),(1,4),(2,4),\cdots,}$, $B_3 = \br{(3,3),(4,3),(5,3)}$, $B_4 = \br{(3,4),(4,4),(4,5),(3,5),\cdots}$. \textbf{(Right): }A smoothed version $(\hat{S}, \hat{w})$ of $(\hat{C},\hat{u})$. There can be more than one unique smoothed version for the same pair $(\hat{C},\hat{u})$; see Properties in~\eqref{eq:propS}-\eqref{eq:propSSum}.
    The pair $(\hat{S}, \hat{w})$ is constructed by iterating over every set $B \in \B$. Every element in $B$ is assigned to a label from the labels of the $4$-coreset points ($8, 17, 21, 22$) as follows: If $|B| > \hat{u}(l_i)$, then $\hat{u}(l_i)$ elements of $B$ are assigned to $l_1$ and $|B| - \hat{u}(l_i)$ are assigned to $l_{i+1}$. If $|B| \leq \hat{u}(l_i)$ then all the elements of $B$ are assigned to $\hat{u}(l_i)$, and $|B|$ is subtracted from $\hat{u}(l_i)$, and so on. If $\hat{u}$ assigns fractional weights, then some elements of $B$ might be assigned to more than one label, as long as the sum of weights over every element in $B$ is $1$. By construction, $(\hat{S},\hat{w})$ satisfies Properties~\eqref{eq:propSNoHoles}-\eqref{eq:propSSum}. Hence, $(\hat{S},\hat{w})$ is a smoothed version of $\hat{C},\hat{u})$. Computing $\cost((\hat{S}, \hat{w}, s)$ can be trivially computed in time only $O(k|\hat{C}|)$ (rather than $O(n)$ where $n$ is the size of the original data), since the sets in the partition $\B$ contain a duplication of a constant number of labels.}
    \label{fig:smoothed_coreset_proof}
\end{figure}

Claim~\ref{claim:oneBlock} now holds by combining cases (i) and (ii) above.

\end{proof}

\textbf{We now prove Lemma~\ref{lemma:coresetQuery}. }
Consider the construction of $(C,u)$ in Algorithm~\ref{Coreset}.
By definition and by Lemma~\ref{lem:ab}, the function $s'$ computed at Line~\ref{lem:ab} of Algorithm~\ref{Coreset} is an $O(k^8\log^2{nm})$-segmentation and satisfies that
\[
\cost(D, s') \in O\left(k\log{n} \cdot \OPT_k(D)\right).
\]
By the last derivation, let $c_\alpha$ be the smallest constant such that
\[
\cost(D, s') \leq c_\alpha \cdot k\log{n} \cdot \OPT_k(D).
\]
By the last inequality and definitions of $\sigma$ and $\alpha$ we obtain that
\begin{equation} \label{eq:SigmaSmallerOPT}
\begin{split}
\sigma & := \frac{\cost(D, s')}{\alpha} \leq \frac{c_\alpha k\log{nm} \cdot \OPT_k(D)}{\alpha}\\
& = \frac{c_\alpha k\log{nm} \cdot \OPT_k(D)}{c_\alpha k \log{nm}}\\ 
&= \OPT_k(D).
\end{split}
\end{equation}

Now consider the partition $\B$ computed at Line~\ref{line:callPartition} of Algorithm~\ref{Coreset} via a call to $\partitionAlg(D,\gamma,\sigma)$. By Lemma~\ref{PTpartition2D}, $\B$ satisfies that
\begin{enumerate}
\item $\OPT_1\left(B\right) \leq \gamma^2\sigma$ for every $B \in \B$.
    \item $\B$ is a partition of $D$ whose size is $|\B| \in O\left(\frac{\alpha}{\gamma^2}\right)$.
    \item There are  $O\left(\frac{k\alpha}{\gamma}\right)$ sub-signals $B \in \B$ where $\left|\br{s(x) \mid (x,y) \in B}\right| > 1$.
    \item $\B$ can be computed in $O(|D|)$ time.
\end{enumerate}

Consider the pair $(C_B,u_B)$ computed at Line~\ref{line:compBlockCoreset} of Algorithm~\ref{Coreset} for some $B \in \B$. 
Now, consider the For loop at Line~\ref{line:forBlockCoreset} of Algorithm~\ref{CoresetQuery}. For every pair $(C_B,u_B)$, there is an iteration of this For loop for which $C' = C_B$. In this iteration, Algorithm~\ref{CoresetQuery} computes a loss  $loss_{C_B}$ that corresponds to $(C_B,u_B)$. We can plug $B$, $(C_B,u_B)$, and $loss_{C_B}$ in Claim~\ref{claim:oneBlock} to obtain that: 
\begin{equation}\label{eq:viaClaimOneBlock}
\begin{cases}
\cost(B, s) = loss_{C_B} & \text{if } $z = 1$\\
\left|\cost(B, s) - loss_{C_B}\right| \leq \varepsilon \cdot \cost(B, s) +  O\left(\frac{\OPT_1(B)}{\varepsilon}\right) & \text{otherwise} 
\end{cases}
\end{equation}

Let $\B_1 \subseteq \B$ contain the set of sub-signals in $\B$ that are not intersected by $s$, i.e., $\B_1 = \br{B \in \B \mid |s(B)| = 1|}$, and let $\B_2 = \B \setminus \B_1$ be the set of sub-signals which are partially intersected by $s$.

By Property (iii) above, 
\begin{equation} \label{eq:sizeB2}
|\B_2| \in O\left(\frac{k\alpha}{\gamma}\right). 
\end{equation}
Furthermore, by combining~\eqref{eq:viaClaimOneBlock} with Property (i) above, for every $B \in \B_2$ we have that
\begin{equation} \label{eq:approxOneBlock}
\begin{split}
& |\cost(B, s) - loss_{C_B}| 
\leq \varepsilon \cdot \cost(B, s) + O\left(\frac{\OPT_1(B)}{\varepsilon}\right)
\leq \varepsilon \cdot \cost(B, s) + O\left(\frac{\gamma^2\sigma}{\varepsilon}\right).
\end{split}
\end{equation}
In other words, the loss $\cost(B,s)$ of every sub-signal $B \in \B_2$ is approximated by $loss_{C_B}$ up to some small error.
Hence, by summing over all $B \in \B_2$ we obtain that
\begin{align}
& \sum_{B \in \B_2} |\cost(B, s) - loss_{C_B}|\nonumber\\ & 
\in \sum_{B \in \B_2} \left(\varepsilon \cdot \cost(B, s) + O\left(\frac{\gamma^2\sigma}{\varepsilon}\right)\right) \label{eq:approxSumB2_1}\\
& \leq \varepsilon \cdot \cost(D, s) + O\left(|\B_2| \cdot \frac{\gamma^2\sigma}{\varepsilon}\right) \label{eq:approxSumB2_2}\\
& \leq \varepsilon \cdot \cost(D, s) + O\left(\frac{k\alpha}{\gamma} \cdot \frac{\gamma^2\sigma}{\varepsilon}\right) \nonumber\\
& \leq \varepsilon \cdot \cost(D, s) + O\left(\frac{k\alpha\gamma}{\varepsilon} \cdot \OPT_k(D)\right) \label{eq:approxSumB2_3}\\
& \leq \varepsilon \cdot \cost(D, s) + O(\varepsilon \cdot \OPT_k(D)) \label{eq:approxSumB2_4}\\
& \in O(\varepsilon \cdot \cost(D,s) ), \label{eq:approxSumB2_5}
\end{align}
where~\eqref{eq:approxSumB2_1} follows from~\eqref{eq:approxOneBlock}, ~\eqref{eq:approxSumB2_2} is by~\eqref{eq:sizeB2}, ~\eqref{eq:approxSumB2_3} is by~\eqref{eq:SigmaSmallerOPT}, \eqref{eq:approxSumB2_4} holds since $k\alpha \gamma \leq \varepsilon^2$, and~\eqref{eq:approxSumB2_5} holds since $\OPT_k(D) \leq \cost(D,s)$ for every $k$-segmentation $s$.

Furthermore, for every $B \in \B_1$, by~\eqref{eq:viaClaimOneBlock} we have that $\cost(B,s) = loss_{C_B}$. Hence, by summing over every $B \in \B_1$ we obtain that
\begin{equation} \label{eq:estimateSumB1}
\sum_{B \in \B_1} \cost(B,s) = \sum_{B \in \B_1} loss_{C_B}.
\end{equation}
In other words, the loss $\cost(B,s)$ of ever sub-signal $B \in \B_1$ is accurately estimated by $loss_{C_B}$.

Algorithm~\ref{CoresetQuery} then outputs the sum of losses 
\begin{equation}\label{eq:defTotalLoss}
loss := \sum{B \in \B} loss_{C_B}. 
\end{equation}
We hence obtain that
\begin{align}
& \left|\cost(D,s) - loss\right| 
= \Bigg| \sum_{B \in \B_1} \left(\cost(B,s) - loss_{C_B}\right) + \sum_{B \in \B_2} \left(\cost(B,s) - loss_{C_B}\right) \Bigg| \nonumber\\
& \leq \left| \sum_{B \in \B_1} \left(\cost(B,s) - loss_{C_B}\right)\right| + \left|\sum_{B \in \B_2} \left(\cost(B,s) - loss_{C_B}\right) \right|
\in O(\varepsilon \cdot \cost(D,s)), \label{eq:DeltaEps}
\end{align}
where the first derivation is by~\ref{eq:defTotalLoss}, second derivation is by the triangle inequality, and the last is by combining~\eqref{eq:estimateSumB1} and~\eqref{eq:approxSumB2_5}.

By~\eqref{eq:DeltaEps}, there is a constant $\Delta \geq 1$ such that
\[
\left|\cost(D,s) - loss\right|  \leq \Delta \varepsilon \cost(D,s).
\]
This concludes the proof of the claim in Lemma~\ref{lemma:coresetQuery}.

\textbf{Computational time: }
Line \ref{line:forBlockCoreset} of the Algorithm \ref{CoresetQuery} is a loop with $\frac{|C|}{|\hat{C}|}$ iterations. Inside this loop: if $z>1$ line \ref{line:ifZEq1} is computed in $O(|\hat{C}|)$, else line \ref{line:forEverySPrimeInS} is another loop with $O(k)$ iterations, inside which the line \ref{line:WhileZGeq1IncrementI} is executed at most $|\hat{C}|$ times and line \ref{line:WhileZGeq1ZeroZ} can be executed only once because it results in $z=0$ and in exiting from the while loop. 

In total the complexity of line 15 is $O(|\hat{C}|)$, of line 11 is $O(k|\hat{C}|)$ and of line 2 and the whole algorithm: $O(k|C|)$

\textbf{Space complexity: }
Algorithm \ref{CoresetQuery} uses only constant amount of additional storage space because in each line of the algorithm only numeric variables are created and variables are reused inside the loops.

\end{proof}

\begin{theorem}[Coreset] \label{theorem:coreset_appendix}
Let $D =\br{(x_1,y_1),\cdots,(x_{N},y_{N})}$ be an $n\times m$ signal i.e., $N := nm$. Let $k\geq 1$ be an integer (that corresponds to the number of leaves/rectangles), and $\varepsilon \in (0,1/4)$ be an error parameter. Let $(C,u)$ be the output of a call to $\coresetAlg(D,k,\varepsilon/\Delta)$ for a sufficiently large constant $\Delta\geq 1$; see Algorithm~\ref{Coreset}. Then, $(C,u)$ is a $(k,\varepsilon)$-coreset for $D$ of size $|C| \in \frac{(k\log(N))^{O(1)}}{\varepsilon^4}$; see Definition~\ref{def:epsCoreset}. Moreover, $(C, u)$ can be computed in $O(kN)$ time.
\end{theorem}
\begin{proof}
To prove that $(C,u)$ is a $(k,\varepsilon)$-coreset for $D$, we need to prove that for every $k$-segmentation $s$, $(C,u)$ suffices to approximate the loss $\cost(D,s)$, up to a multiplicative factor of $1+\varepsilon$, in time that depends only on $|C|$ and $k$. 

Let $s$ be a $k$-segmentation and let $loss \geq 0$ be an output of a call to $\coresetQuery((C,u),s)$; see Algorithm~\ref{CoresetQuery}. Then, by Lemma~\ref{lemma:coresetQuery}, $loss$ can be computed in $O(k|C|)$ time (i.e., in time that depends only on $|C|$ and $k$), and provides, as required, a $(1+\varepsilon)$-approximation to $\cost(D,s)$ as
\[
\left|\cost(D,s) - loss\right| \leq \Delta \cdot \frac{\varepsilon}{\Delta} \cdot \cost(D,s) = \varepsilon \cdot \cost(D,s).
\]
Hence, $(C,u)$ is a $(k,\varepsilon)$-coreset for $D$.

Line~\ref{line:callBic} of Algorithm~\ref{Coreset} can be computed in $O(k\cdot |D|)$ time by Lemma~\ref{lem:ab}. Line~\ref{line:callPartition} of Algorithm~\ref{Coreset} can be computed in $O(|D|)$ time by Lemma~\ref{PTpartition2D}. The loop at Line~\ref{line:everyBlock} can be computed in $\sum_{B\in \B} O(|B|) = O(|D|)$ time by Section~\ref{sec:cara}.
Hence, the call $\coresetAlg(D,k,\varepsilon)$ can be implemented in $O(k|D|) = O(kmn)$ time.

\textbf{Proof behind Line~\ref{Line:replaceCoords}. }We now prove that the replacements of the coordinates applied at Line~\ref{Line:replaceCoords} does not violate the correctness of the algorithm.
Observe that replacing the coordinates of entries inside each cell, while keeping the same labels, does not affect the variance of this subset. Therefore, the cost of this cell, which is computed in Algorithm~\ref{CoresetQuery}) and depends only on the labels, remains exactly the same.

\textbf{Space complexity: }By construction, each pair $(C_B,u_B)$ computed at Line~\ref{line:compBlockCoreset} can be stored using only $O(1)$ space. Hence, the concatenation $(C,u)$ of the $|B|$ pairs $\br{(C_B,u_B) \mid B \in \B}$ can be stored using $O(|\B|) \in O(\alpha/\gamma^2) = O(\alpha (\beta k)^2/\varepsilon^4) =  O\left(\frac{k^{O(1)}\log^{O(1)}nm}{\varepsilon^4}\right)$ space.
\end{proof}

\section{The Caratheodory Theorem} \label{sec:cara}
Given a point $p\in \REAL^d$ inside the convex hull of a set of points $P \subseteq \REAL^d$, Caratheodory’s Theorem proves that there is a subset of at most $d+1$ points in $P$ whose convex hull also contains $p$. 

\begin{theorem}[Caratheodory's Theorem~\cite{caratheodory1907variabilitatsbereich, nasser2020autonomous}]\label{theorem:cara}
Let $P \subseteq \REAL^d$ be a (multi)set of $n$ points. Then in $O(nd^3)$ time we can compute a subset $Q \subseteq P$ and a weights functions $u:Q\to [0,\infty)$ such that: (i) $Q\subseteq P$, (ii) $|Q| = d+1$, (iii) $\sum_{q \in Q}u(q) \cdot q = \frac{1}{n}\sum_{p \in P} p$, and (iv) $\sum_{q \in Q}u(q) = n$.
\end{theorem}

\begin{corollary} \label{cor:cara}
Let $D$ be an $n\times m$ sub-signal. Then, in $O(|D|)$ time we can compute a weighted $n\times m$ sub-signal $(A,w)$ such that: (i) $A \subseteq D$, (ii) $|A| = 4$, (iii) $\displaystyle \sum_{(a,b) \in A} w(a,b)\cdot (b \mid b^2 \mid 1) = \sum_{(x,y) \in D}(y \mid y^2 \mid 1)$, and $\displaystyle \sum_{(a,b) \in A} w(a,b) = |D|$.
\end{corollary}
\begin{proof}
Define the multi-set $P = \br{(y \mid y^2 \mid 1) \in \REAL^3 \mid (x,y) \in D}$. Now, substituting $P$ in Theorem~\ref{theorem:cara} yields that in $O(n)$ time we can compute a subset $Q \subseteq P$ and a weights functions $u:Q\to [0,\infty)$ such that: (i) $Q\subseteq P$, (ii) $|Q| = 4$, (iii) $\sum_{q \in Q}u(q) \cdot q = \frac{1}{|D|}\sum_{p \in P} p$, and (iv) $\sum_{q \in Q}u(q) = |D|$. 

Now, add to $A$ a single element $(x,y) \in D$ for every $(y \mid y^2 \mid 1) \in Q$. In other words, for every element chosen for the set $Q$ by the Caratheodory theorem, add its corresponding element from $D$ to $A$. Furthermore, define $w(x,y) = |D| \cdot u((y \mid y^2 \mid 1))$ for every $(x,y) \in A$.
Corollary~\ref{cor:cara} trivially holds for $(A,w)$.
\end{proof}

\end{document}